\documentclass{article}

\usepackage[square,numbers,sort&compress]{natbib}
\usepackage[final]{neurips_2025}
\usepackage{amssymb}
\usepackage{amsmath}
\usepackage{amsthm}
\usepackage{graphicx}
\usepackage{subcaption}
\usepackage{mathtools}

\def\eps{\varepsilon}

\usepackage[utf8]{inputenc} % allow utf-8 input
\usepackage[T1]{fontenc}    % use 8-bit T1 fonts
\usepackage{hyperref}       % hyperlinks
\usepackage{url}            % simple URL typesetting
\usepackage{booktabs}       % professional-quality tables
\usepackage{amsfonts}       % blackboard math symbols
\usepackage{nicefrac}       % compact symbols for 1/2, etc.
\usepackage{microtype}      % microtypography
\usepackage{xcolor}         % colors

\newtheorem{theorem}{Theorem}[section]
\newtheorem{remark}{Remark}[section]

\newtheorem{proposition}{Proposition}[section]

\newtheorem{assumption}{Assumption}

\usepackage{enumitem}
\setlist[itemize]{leftmargin=.4cm}

\title{%A Theoretical Analysis of 
Preconditioned Langevin Dynamics with Score-Based Generative Models for Infinite-Dimensional Linear Bayesian Inverse Problems}

\author{%
  Lorenzo Baldassari \\
  University of Basel \\
  \texttt{\texttt{ lorenzo.baldassari@unibas.ch}} \\
  \And
  Josselin Garnier\\
  Ecole Polytechnique, IP Paris\\
  \texttt{ josselin.garnier@polytechnique.edu}\\
  \And
  Knut S\o{}lna\\
  University of California Irvine\\
  \texttt{ksolna@uci.edu}\\
  \And
  Maarten V. de Hoop\\
  Rice University\\
  \texttt{mvd2@rice.edu}\\
}
\begin{document}

\maketitle

\begin{abstract}

Designing algorithms for solving high-dimensional Bayesian inverse problems directly in infinite‑dimensional function spaces---where such problems are naturally formulated---is crucial to ensure stability and  convergence  
% avoid theoretically implausible behaviors
as the discretization  of the underlying  problem is refined. In this paper, we contribute to this line of work by analyzing a widely used sampler for linear inverse problems: Langevin dynamics driven by score‑based generative models (SGMs) acting as priors, formulated directly in function space. Building on the  theoretical framework for SGMs in Hilbert spaces, we give a rigorous definition of this sampler in the infinite-dimensional setting and derive, for the first time, error estimates that explicitly depend on the approximation error of the score. As a consequence, we obtain sufficient conditions for global convergence in Kullback–Leibler divergence on the underlying function space. Preventing numerical instabilities requires preconditioning of the Langevin algorithm and we prove the existence  and the form of an optimal preconditioner. The preconditioner depends on both the score error and the forward operator and guarantees a uniform convergence rate across all posterior modes. Our analysis applies to both Gaussian and a general class of non‑Gaussian priors. Finally, we present examples that illustrate and validate our theoretical findings.

\end{abstract}

\section{Introduction}

Inverse problems arise in many challenging applications, such as X-ray computed tomography, seismic tomography, inverse heat conduction, and inverse scattering. These problems share a common goal: to estimate unknown parameters from noisy observations or measurements \cite{tarantola2005inverse}. What makes them  difficult is that they are often ill-posed  in the sense of Hadamard \cite{hadamard2003lectures}: they may have multiple solutions, no solutions at all, or solutions that are highly sensitive to small perturbations in the data. A possible approach to address these difficulties is to cast the problem in a probabilistic framework known as Bayesian inference.
%, which allows for a full characterization of all possible solutions and their relative uncertainties \citep{ tarantola2005inverse, lehtinen1989linear, stuart2014uncertainty}
In the Bayesian approach, one first specifies a prior distribution that encodes knowledge about the unknown before any data is observed, along with a model for the observational noise. Bayes' rule is then used to update this prior knowledge in light of the measurements, yielding the so-called \emph{posterior distribution}, which describes the distribution of the unknown conditioned on the data. By sampling from the posterior one can extract statistical information and quantify uncertainty in the solution \cite{stuart2010inverse, knapik2011bayesian, dashti2011uncertainty, stuart2014uncertainty}.

A central challenge in applying Bayesian inference to inverse problems is that in many cases---especially those governed by partial differential equations (PDEs)---the unknowns to be estimated are \emph{functions} that lie in a suitable function space, typically an infinite-dimensional Hilbert space. It is therefore crucial to design Bayesian inference algorithms that are both theoretically sound and computationally effective in arbitrarily high dimensions.  A way to achieve this is by lifting these problems to an infinite-dimensional space and designing inference methods directly in that setting. This approach, sometimes referred to as ``\emph{apply-algorithm-then-discretize}''---or, in the context of Bayesian inference, ``\emph{Bayesianize-then-discretize}''---allows for the development of algorithms that are inherently discretization-invariant, as the Bayes formula and algorithms are properly defined on Hilbert spaces \cite{stuart2010inverse, dashti2013bayesian}. In contrast, the opposite approach---``\emph{discretize-then-Bayesianize}''---can lead to several issues, such as instability as the discretization of the underlying problem is refined, or worse, methods that seem stable but whose results are theoretically implausible \cite{cotter2013mcmc, lassas2004can}.

These considerations manifest clearly even in simple scenarios. In Figure \ref{fig:langevin_challenges} we consider two examples involving a vanilla diffusion Langevin sampler. In the first one, we sample from a Gaussian posterior. While the method appears numerically stable and produces samples with seemingly reasonable behavior, a closer inspection shows that the samples carry infinite energy---they do not belong to the infinite-dimensional Hilbert space. That is, the algorithm is producing objects that are not valid functions in the limit of refined discretization. In the second example, we attempt to fix this by choosing a \emph{trace-class} prior, which ensures that samples have finite energy and are well-defined in a Hilbert space. This theoretically-motivated structure, however, comes at a cost: without adjustments, the drift of the vanilla Langevin sampler may diverge at fine scales.

\begin{figure}[h]
    \centering
    \includegraphics[width=\textwidth]{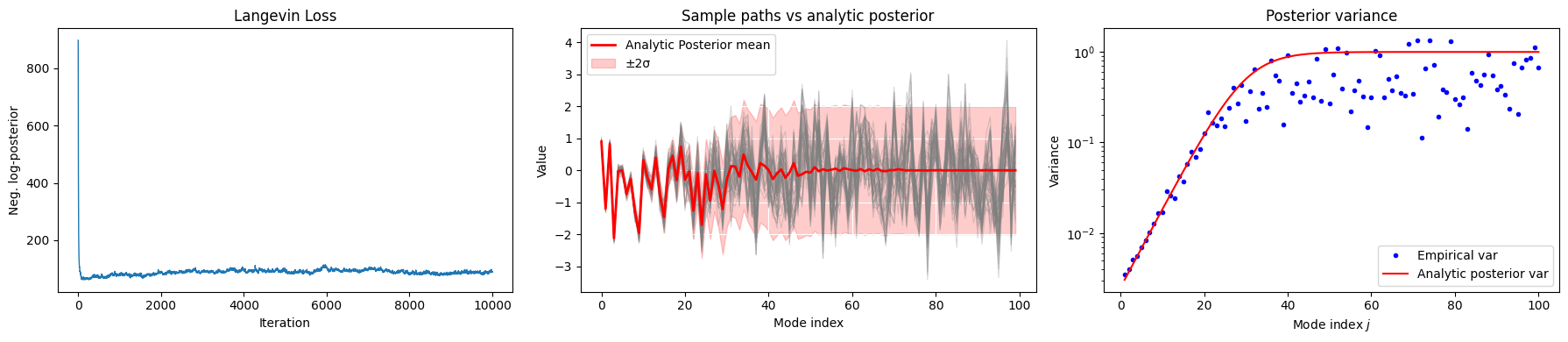}
    \includegraphics[width=\textwidth]{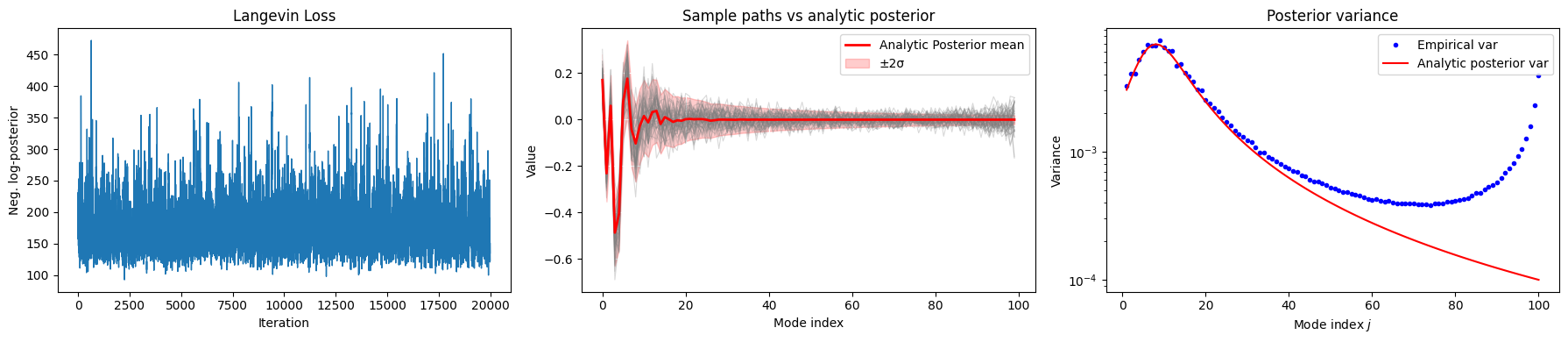}
\caption{We consider the toy linear inverse problem $y_j= A_{jj} X^{(j)}_0 + n_j$ in the basis $(v_j)$ of the Hilbert space $H$, with $A_{jj} = e^{-0.1 j}$ and $n_j \sim \mathcal{N}(0,0.05^2)$, for $j\leq 100$. In the top row, we sample the posterior using an identity prior covariance on $X_0$. The Langevin diffusion seems stable, but the eigenvalues of the posterior covariance do not decay at infinity and therefore the samples do not belong to the Hilbert space $H$. In the bottom row, we use a trace-class covariance prior with diagonal terms $\sim 1/j^2$; the drift of the vanilla Langevin sampler starts diverging at fine scales.}
    \label{fig:langevin_challenges}
\end{figure}

These types of challenges, intrinsic to the infinite-dimensional setting, have long been studied in the Bayesian inverse problems community, but are now receiving renewed attention with the rise of deep learning methods for posterior sampling. One popular class of methods that still lacks a complete theoretical understanding in this context is \emph{score-based generative models} (SGMs), which generate samples from complex distributions by first learning the {(Stein) score}---the gradient of the log-density \cite{liu2016kernelized}---and then using it in various sampling algorithms \citep{vincent2011connection, song2020improved}. For example, \cite{song2019generative} employs the learned score in a Langevin-based sampler, while \cite{song2020score} unified SGMs and diffusion-based methods \cite{sohl2015deep, ho2020denoising} through a stochastic differential equation (SDE) framework, known as {score-based diffusion models}.  
%In this setup, data is noised continuously through a forward diffusion process, while sampling is performed via a reverse-time SDE whose drift depends on the score, learned via neural networks \citep{vincent2011connection, song2020improved}. 
After their introduction, SGMs have been applied successfully to Bayesian inverse problems, either by learning the score conditioned on data \cite{batzolis2021conditional, kawar2021snips, jalal2021robust}, or by using the score of the prior distribution---the \emph{unconditional score model}---within Langevin-type samplers. Crucially, with a few exceptions \cite{baldassari2024conditional, baldassari2024taming, hosseini2023conditional, hagemann2023multilevel}, these approaches assume that the posterior is supported in a finite-dimensional space, leaving the challenges of infinite dimensions to heuristics and ad hoc solutions.
 
 In this work, we present a detailed analysis of SGMs for Bayesian inference of linear inverse problems, going beyond the common assumption that the posterior is supported on a finite-dimensional space. We focus on a widely used posterior sampling technique that combines SGMs---used as powerful learned priors to capture complex features---with a Langevin-based sampler \cite{feng2023score, sun2023provable, xu2024provably}. Lifting the problem directly to function spaces is not a mere technicality: we show that, to provably sample the posterior, the Langevin diffusion must be modified by a preconditioning operator $C$ acting on the Hilbert space. This preconditioner is not an ad hoc fix but rather %structurally  
 built into the fabric of the infinite-dimensional setting: it first appears in the forward diffusion process \eqref{eq:forward-diffusion} whose time-reversal is used to learn the prior, and must then be carried through into the Langevin sampler to ensure convergence to the correct posterior (Section \ref{sec:gaussian}). Crucially, $C$ cannot be the identity: for the time-reversed diffusion to remain Hilbert-space-valued, %its driving $C$-Wiener process must itself take values in the Hilbert space---that is, 
 $C$ must be trace class. Setting $C=I$ leads to the same theoretical and numerical issues as highlighted in Figure \ref{fig:langevin_challenges} above. 
 %Preconditioning with a trace-class $C$, instead, both preserves the well-posedness of the underlying diffusion and stabilizes the Langevin diffusion across all modes.
 
The importance of preconditioning in function spaces has been well established in the context of posterior sampling \cite{stuart2004conditional, hairer2005analysis, hairer2007analysis, stuart2010inverse, cotter2013mcmc, beskos2017geometric}, but its implications have not yet  fully explored for infinite-dimensional SGMs. In this setting, we characterize the interplay between the preconditioner $C$, the trace-class prior, the score approximation error, and the linear forward map of the inverse problem. In particular, we analyze the impact of the score approximation error at small times---where the score is learned in practice---and identify a \emph{theoretically optimal preconditioner} that ensures uniform convergence rates across posterior modes (Section \ref{sec:preconditioning}). We carry out the analysis by focusing on two cases: a Gaussian prior measure, and a more general class of priors which are absolutely continuous with respect to a Gaussian measure (Section \ref{sec:non-gaussian}). Illustrations are provided in Section \ref{sec:illustrations}.

{\bf Related Work. \hspace{0.5mm}} 
%
%By analyzing a widely used sampler for linear inverse problems---Langevin dynamics driven by score-based generative models (SGMs)---directly in infinite dimensions, our work combines elements from three contemporary research areas: MCMC methods, Bayesian inverse problems, and SGMs.
%
 There exists a large body of literature on infinite-dimensional MCMC algorithms \citep{stuart2004conditional, hairer2005analysis, hairer2007analysis, beskos2017geometric, wallin2018infinite, stuart2010inverse, durmus2019high, durmus2017nonasymptotic, dalalyan2017theoretical, hairer2014spectral, cotter2013mcmc, cui2016dimension, cui2024multilevel, beskos2018multilevel, morzfeld2019localization, beskos2008mcmc}, which include a variety of preconditioning strategies for posterior sampling. However, these works precede the recent wave of papers on SGMs and therefore do not address the central focus of our analysis: the interplay between the score approximation error, the preconditioning operator, the trace-class prior, and the sampler convergence, which we study in detail in both the Gaussian and non-Gaussian settings. 

Closest to our work are the papers that use SGMs for posterior sampling, such as \cite{feng2023score, sun2023provable, xu2024provably, baldassari2024taming}, which employ SGMs as learned priors in a Langevin-type diffusion algorithm. Among these works, the theoretical analysis of \cite{sun2023provable} is the most directly related to ours. However, there are key differences. \cite{sun2023provable} analyze Langevin dynamics with SGMs for posterior sampling in finite dimensions, as their results provide convergence error estimates that explicitly depend on the score approximation error but diverge as the dimension of the problem increases. In contrast, our error analysis, since it is formulated directly in infinite dimensions, provides conditions to ensure global boundedness (Theorem \ref{thm:error-analysis}). Moreover, the finite-dimensional setting of \cite{sun2023provable} does not address the role of preconditioning, which becomes essential in infinite dimensions. 
Other related works include \cite{ma2022accelerating, ma2025preconditioned}
%, by the same authors, 
which investigate preconditioning in Langevin dynamics with SGMs.  However, these analyses are also finite-dimensional and do not account for the score approximation error. As a result, they do not capture the critical role of preconditioning, which---as we show in Section \ref{sec:preconditioning}---becomes crucial in function spaces.

As we have pointed out several times, the learned score plays a key role in our analysis. Among the theoretical frameworks defining 
SGMs in infinite dimensions \cite{franzese2024continuous, franzese2025generative, lim2023score, kerrigan2022diffusion, lim2024score, hagemann2023multilevel, bond2023infty}, we follow those of \cite{pidstrigach2024infinite, baldassari2024conditional} for continuous-time diffusions. An important contribution of our work is to show that the convergence bound depends explicitly on the accuracy of the approximated score, and that controlling this error is key to designing a preconditioner that ensures convergence in function spaces (Theorem~\ref{thm:preconditioner}).

Finally, we note that \cite{baldassari2024taming} also explores the role of preconditioning to ensure convergence in infinite dimensions in the context of SGMs. Their analysis is conducted in a more complex setting---nonlinear inverse problems. Their argument builds on the proof of \cite{sun2023provable}, but the difficulties of the nonlinear setting prevent them from identifying an optimal preconditioner. In contrast, our work takes full advantage of the linear setting, where the distributions at play admit explicit formulas. This allows us to derive detailed error estimates in the small diffusion time regime, where the score is typically learned, and discuss the impact of the score approximation error on posterior sampling---including the effects of preconditioning on the bias error. Furthermore, their analysis focuses only on Gaussian priors, while we generalize and consider non-Gaussian priors (Section \ref{sec:non-gaussian}).

\section{Langevin Posterior Sampling with Score-Based Generative Priors}\label{sec:langevin-SGM}
We work in the setting of a linear Bayesian inverse problem formulated in infinite dimension. Let $H$ be a separable Hilbert space with inner product $\langle \cdot \,,  \cdot \rangle$, and let $C,C_\mu:H\to H$ be trace-class,  positive-definite, symmetric
covariance operators. The unknown quantity of interest is an $H$-valued random variable $X_0 \sim \mu$, where the prior measure $\mu$ is assumed to be absolutely continuous with respect to a Gaussian reference measure $\mathcal{N}(0,C_\mu)$, with density
\begin{equation}
\frac{d\mu}{d{\cal N}(0,C_\mu)}(X) \propto \exp \big(-\Phi(X) \big).
\label{def:density-prior}
\end{equation}
The observations $y\in \mathbb{R}^N$ are modeled as
\[
y = A X_0 + n,
\]
where $A:H\to \mathbb{R}^N$ is a linear operator, and $n \sim \mathcal{N}(0,\sigma^2I_N)$ is Gaussian observational noise independent of $X_0$. Since we consider an observational model corresponding to observing a finite-dimensional subspace of $H$, there exists an orthonormal basis $(v_j)$ of $H$ such that $A v_j = 0$ for all $j>N$. Let $(e_j)$ denote the standard basis of $\mathbb{R}^N$. Then the observation model can be written as 
$y_i= \sum_{j=1}^N A_{ij} X_0^{(j)} + n_i$, where $A_{ij} = \langle e_i, A\,v_j \rangle$, $y_i = \langle y,e_i\rangle$,  $X^{(j)}_0 = \langle X_0 ,v_j\rangle $, and $n_i = \langle n,e_i\rangle$. 

The posterior distribution $\pi_y$ of $X_0$ conditioned on the observations $y$ is absolutely continuous with respect to $\mathcal{N}(0,C_\mu)$:
\begin{equation}
\frac{d \pi_y}{d \mathcal{N} (0,C_\mu)} (X) \propto \exp\Big(- \Phi(X) - \frac{1}{2\sigma^2} \|AX - y\|^2\Big) .
\label{eq:posterior}
\end{equation}
The goal of infinite-dimensional Bayesian inference is to design sampling methods for $\pi_y$ whose performance remains stable as the underlying discretization is refined. To this end, we study a widely used sampler---a Langevin-type diffusion driven by score-based generative priors---this time formulated directly in infinite dimensions rather than in the usual finite-dimensional setting. In particular, we consider the continuous-time SDE
\begin{equation}
dX_t =  S_\theta (X_t, \tau;\, \mu) dt +C \nabla_X \log \rho (y-A X_t) dt +\sqrt{2 C} dW_t,
\label{eq:Langevin-SDE}
\end{equation}
where $\rho$ is the noise density, $C$ acts as a preconditioner, $\nabla_X$ denotes the Fr\'echet derivative with respect to $X$, $W_t$ is a Wiener process on $H$, and $S_\theta(X_t,\tau;\, \mu)$ is a neural network approximation of the score function
\begin{equation}
S (X, \tau; \, \mu)= -(1-e^{-\tau})^{-1} (X-e^{-\tau/2}\mathbb{E}[X_0|X_\tau=X]),
\label{def:score}
\end{equation}
which corresponds to the drift term in the time-reversed SDE of the Hilbert-space-valued forward diffusion
\begin{equation}
dX_\tau = -\frac{1}{2}X_\tau d\tau + \sqrt{C} dW_\tau, \qquad X_0 \sim \mu.
\label{eq:forward-diffusion}
\end{equation}
There are two important aspects to note here. First, both the Langevin SDE \eqref{eq:Langevin-SDE} and the forward diffusion \eqref{eq:forward-diffusion} are driven by a $C$-Wiener process, where $C$ is trace-class, which is crucial for ensuring that the samples are supported on the Hilbert space. Most of the technical difficulties in infinite dimensions arise from this. Second, although the score is often expressed as $\nabla \log p_\tau$ in finite-dimensional settings, the density $p_\tau$ is not defined in infinite dimensions, since a Lebesgue reference measure does not exist. For this reason, in the following we adopt the conditional expectation representation of the score---or,  more precisely, an equivalent formulation derived from it, as stated in the next proposition, which was first proved in \cite{pidstrigach2024infinite}.
\begin{proposition}
The score \eqref{def:score} can be written as
    \begin{equation}
    S(X, \tau;\,\mu) = - e^{\tau/2} \mathbb{E} [C (C_\mu C_\tau^{-1})^{-1} \nabla \Phi(X_0) \mid X_\tau = X] - C C_\tau^{-1}X,
    \label{def:score-non-gaussian}
    \end{equation}
    where $C_\tau = e^{-\tau} C_\mu + (1-e^{-\tau})C$.
\end{proposition}

%\begin{remark}
%If $\mu$ is a Gaussian measure (i.e., $\Phi =0$), then $S(X, \tau;\, \mu) = - C C_\tau^{-1}X$.
%\end{remark}

The idea behind samplers like \eqref{eq:Langevin-SDE} is simple yet powerful. By training $S_\theta(X,\tau;\mu)$ to approximate $S (X, \tau; \, \mu)$, one can effectively learn potentially complex priors $\mu$---since, once $S_\theta(X,\tau;\mu)$ is known, one can sample from $\mu$ by simulating the backward-in-time dynamics---and then incorporate such priors within a Langevin sampling scheme. What remains less understood, however, is how this approach extends to the infinite-dimensional setting, particularly in relation to the error introduced by approximating the score and whether the sampler remains stable. 
In the sections that follow, we address this gap---we prove convergence of \eqref{eq:Langevin-SDE} to the correct posterior and derive error bounds, along with conditions ensuring a globally bounded convergence error. We also elucidate the role of the preconditioner $C$. 
Our analysis is divided into two parts: one addressing the case of Gaussian priors, and the other the non-Gaussian case.%, each examining how the score approximation impacts both the error analysis and the choice of preconditoner.

\section{Error Analysis in the Gaussian Setting}\label{sec:gaussian}
We begin our analysis of the continous-time Langevin SDE \eqref{eq:Langevin-SDE} in the infinite-dimensional setting by examining the case where the prior of $X_0$ is a Gaussian measure. While this case may seem to defeat  the purpose of using a score-based generative model to learn a simple prior, it  provides illuminating insights, as it allows us to detail the impact of the score approximation error on the stationary distribution of \eqref{eq:Langevin-SDE}, offers a clear interpretation of the infinite-dimensional difficulties, and paves the way for the derivation of an explicit form of the optimal preconditioner (Section \ref{sec:preconditioning}).

We assume in this section that $\Phi=0$. The posterior \eqref{eq:posterior} is Gaussian:
\[
\pi_y  = \mathcal{N}\left(\left[ C_\mu^{-1} +  \sigma^{-2}A^\top A\right]^{-1}  \sigma^{-2} A^\top y,  \left[ C_\mu^{-1} +  \sigma^{-2}A^\top A\right]^{-1}\right).
\]
%Let $\widetilde{X}_0 = \left[{X}_0^{(1)}, {X}_0^{(1)}, \ldots, {X}_0^{(N)}\right]$. 
We also assume that both $C$  and $C_\mu$ are diagonal in the basis $(v_j)$, with eigenvalues $(\lambda_j)$ and $(\mu_j)$, respectively. We can make a few remarks:
\vspace*{-0.6em}
\begin{itemize}
\item In the $(v_j)$ basis, the posterior decomposes into a Gaussian $\pi_y^N$ on the span of the first $N$ observed modes and a product of marginal Gaussian over the unobserved modes $j>N$. 
\item For the observed modes---i.e., those $j\leq N$ influenced by the data through the forward operator $A$---the distribution is
\[
\hspace*{-0.1in}
\pi_y^{N} = \mathcal{N} \left( \! \left[ 
%\operatorname*{Diag}\limits_{1\leq j \leq N}\left(\mu_j^{-1}\right) 
C_{\mu,N}^{-1} 
+  \sigma^{-2} A_N^\top A_N \right]^{-1} \!\! \sigma^{-2}  A_N^\top y ,  \left[ 
%\operatorname*{Diag}\limits_{1\leq j \leq N}\left(\mu_j^{-1}\right) 
C_{\mu,N}^{-1}
+  \sigma^{-2}  A_N^\top A_N \right]^{-1}\right),
\mbox { with }
C_{\mu,N}\! = \!\! \operatorname*{Diag}\limits_{1\leq j \leq N}\left(\mu_j\right) .
\]
%with $C_{\mu,N}=\operatorname*{Diag}\limits_{1\leq j \leq N}\left(\mu_j\right)$.
\item For the unobserved modes $j> N$, which lie in the nullspace of $A$, the posterior coincides with the prior: $\pi_y^{(j)} = \mathcal{N} \left(0 ,  \mu_j \right)$.
\item The score function is $ S(X,\tau; \, \mu) = - \sum_j s_j(\tau;\mu) X^{(j)} v_j$,
with $s_j(\tau;\mu)=\frac{\lambda_j}{e^{-\tau} \mu_j+(1-e^{-\tau}) \lambda_j}$.
\end{itemize}
\vspace*{-0.6em}
The block diagonalization of the system by $(v_j)$ justifies the following assumption on the form of the score approximation error. 

\begin{assumption}\label{assumption:score-error}
We consider an  approximate score $S_\theta(X,\tau; \, \mu)$ such that 
\[
\big\langle S(X,\tau;\, \mu)-S_\theta(X,\tau;\, \mu), \; v_j\big\rangle = \eps_j^a(\tau) X^{(j)}  + \eps^b_j(\tau).
\]    
\end{assumption}
Define $X^N= \sum_{j=1}^N X^{(j)} v_j$ and similarly let $W_t^N$ denote the projection of $W_t$ onto the first $N$ modes. By Assumption \ref{assumption:score-error}, for the observed modes $j\leq N$, the preconditioned Langevin dynamics \eqref{eq:Langevin-SDE} become
\begin{align*}
dX_t^N =& -\left[ \operatorname*{Diag}\limits_{1\leq j \leq N}\left( s_j(\tau;\mu)\right) +  \operatorname*{Diag}\limits_{1\leq j \leq N}\left(\eps^a_j(\tau)\right) +   
%\operatorname*{Diag}\limits_{1\leq j \leq N}\left(\lambda_j\right) 
C_N
\sigma^{-2} A_N^\top A_N \right] X_t^N dt\\
&
+ \left[  
%\operatorname*{Diag}\limits_{1\leq j \leq N}\left(\lambda_j\right) 
C_N
\sigma^{-2} A_N^\top y -  \operatorname*{Diag}\limits_{1\leq j \leq N}\left(\eps_j^b(\tau)\right) \right] dt +
%\sqrt{2}  \operatorname*{Diag}\limits_{1\leq j \leq N}\left(\lambda_j^{1/2}\right) 
 \sqrt{2C_N} dW^N_t, 
% \mbox{ with }  C_N= \operatorname*{Diag}\limits_{1\leq j \leq N}\left(\lambda_j\right),
\end{align*}
with $ C_N= \operatorname*{Diag}\limits_{1\leq j \leq N}\left(\lambda_j\right)$. For the unobserved modes $j>N$, we have
\begin{align*}
dX_t^{(j)} =& -\left[ s_j(\tau;\mu) +   \eps^a_j(\tau)\right] X^{(j)}_t dt
-    \eps_j^b(\tau) dt +\sqrt{2 
 \lambda_j} dW^{(j)}_t. 
\end{align*}
We are now ready to derive the stationary distribution of the continuous-time SDE \eqref{eq:Langevin-SDE}. The following proposition makes explicit the dependence on the score approximation error; its proof is included in Appendix \ref{appendix:prop-3-1}.
\begin{proposition}\label{prop:stationary-langevin}
The stationary distribution $\check{\pi}_y$ of the preconditioned Langevin diffusion with approximate score in the drift term is the Gaussian measure 
%$\check{\pi}_y = \mathcal{N}(\check{m}(\tau),\check{v}(\tau))$.  
with mean $\check{m}(\tau)=(\check{m}^N(\tau) , (\check{m}_j(\tau))_{j\geq N+1})$
and covariance $\check{v}(\tau)=\check{v}^N(\tau) \oplus \operatorname*{Diag}\limits_{ j \geq N+1} \left(\check{v}_j(\tau) \right) $.
%Denote by $\check{\pi}_y^{(j)} =   \mathcal{N}\left( \check{m}_j(\tau),  \check{v}_j(\tau) \right)$ its $j$-th mode marginals in the orthonormal basis $(v_j)$ . 
For the observed modes $j\leq N$, we have
\begin{align}
\label{expres:vn}
&\check{v}^N (\tau) = \left[ 
%\operatorname*{Diag}\limits_{1\leq j \leq N}\left( \frac{1}{e^{-\tau}\mu_j + (1-e^{-\tau})\lambda_j }\right) 
C_{N}^{-1} \operatorname*{Diag}\limits_{1\leq j \leq N}\left( s_j(\tau;\mu) \right) + \sigma^{-2} A_N^\top A_N + C_N^{-1} \operatorname*{Diag}\limits_{1\leq j \leq N} \left(  \eps_j^a(\tau)\right)  \right]^{-1},
\\
&
\check{m}^N(\tau)  = \check{v}^N(\tau)  \left[ \sigma^{-2} A_N^\top y  - C_N^{-1} \operatorname*{Diag}\limits_{1\leq j \leq N}\left(   \eps_j^b(\tau)\right) \right]
,
\label{expres:mn}
\end{align}
%with  $C_{\tau,N} = \operatorname*{Diag}\limits_{1\leq j \leq N}\left(  e^{-\tau}\mu_j + (1-e^{-\tau})\lambda_j \right)$,
while for the unobserved modes $j>N$, we have
\begin{equation}
\label{expres:vjmj}
\check{v}_j(\tau) = \left[ 
%\frac{1}{e^{-\tau} \mu_j + (1-e^{-\tau})\lambda_j}  
 \lambda_j^{-1}s_j(\tau;\mu)
+ \lambda_j^{-1} \eps_j^a(\tau) \right]^{-1},
\qquad \check{m}_j(\tau) = -\check{v}_j(\tau)  \lambda_j^{-1} \eps_j^b(\tau).
\end{equation}
\end{proposition}

Based on Proposition \ref{prop:stationary-langevin}, we make a few comments:

\begin{itemize} 
\item If we have access to the perfect score, that is, $\eps_j^a =\eps_j^b=0$ for all $j$, then
\begin{align*}
\check{m} (\tau) &=\Big[ C_{\tau}^{-1} +\sigma^{-2} A^\top A \Big]^{-1}
\sigma^{-2} A^\top y   
\stackrel{\tau \to 0}{\to}
\left[ C_\mu^{-1}+\sigma^{-2}A^\top A \right]^{-1} \sigma^{-2} A^\top y,   \\
\check{v} (\tau) &= \Big[ C_\tau^{-1} +\sigma^{-2} A^\top A \Big]^{-1} 
\stackrel{\tau \to 0}{\to}
\left[ C_\mu^{-1} + \sigma^{-2}A^\top A \right]^{-1}.
\end{align*}
That is, we recover the posterior $\pi_y$ given the data. It does not depend on the preconditioner $C$.
\item The error $\eps_j^a$ can have an impact on the stationary distribution of $X^{(j)}_t$, but as long as it is smaller than $\lambda_j/\mu_j$ (i.e., the relative error in the approximation of the score is small), the impact is small. 
\item
The error $\eps_j^b$ can induce a bias. The bias can be large because the mean of the $j$-th mode marginal of $\check{\pi}_y^{(j)}$ is amplified by $\lambda_j^{-1}$. The preconditioner cannot prevent from this bias.
\end{itemize}

We can make our analysis more quantitative by presenting mode-by-mode and global convergence error estimates for the preconditioned Langevin sampler in the Gaussian setting.
To simplify the discussion, the following theorem is stated by assuming that $A_N^\top A_N$ is diagonal.
 
\begin{theorem}\label{thm:error-analysis}
We define $p_j = \lambda_j /\mu_j$ for all $j$. Let  $\check{\pi}_y^{(j)}$ and $\pi_y^{(j)}$  denote the $j$-th mode marginals of the approximate and true posterior distributions, $\check{\pi}_y$ and $ {\pi}_y $, respectively. Suppose that 
$p_j^{-1} \eps_{j}^a(\tau)  = O(\tau)$, $\lambda^{-1}_j \eps_{j}^b(\tau) = O(1)$.
Then, for $j\leq N$, the Kullback-Leibler divergence satisfies
\begin{align}
\nonumber
& \textup{D}_\textup{KL}\Big(\check{\pi}^{(j)}_y \,\Big|\Big|\, {\pi}^{(j)}_y \Big) 
=   \frac{1}{2}\lambda_j^{-2} \eps_{j}^b(\tau)^2 \\
& %=   \frac{1}{2}\lambda_j^{-2} \eps_{j}^b(\tau)^2 
\quad -  \frac{\lambda_j^{-1} \eps_{j}^b(\tau)}{1+\sigma^{-2}\mu_j 
%\sum_{i=1}^N A_{ij}^2 
(A_N^\top A_N)_{jj}
}
\left( \sigma^{-2} 
%\sum_{i=1}^N A_{ij} y_i  
(A_N^\top y)_j - \lambda_j^{-1} \eps_{j}^b(\tau) \right) \left( (p_j-1)\tau -p_j^{-1} \eps_{j}^a(\tau) \right)  + O(\tau^2).
\label{eq:thm-3-1}
\end{align}
For $j> N$, we have
$\textup{D}_\textup{KL}\big(\check{\pi}^{(j)}_y \,\big|\big|\, {\pi}^{(j)}_y \big)  =  \lambda_j^{-2}\eps_{j}^b(\tau)^2 \left(\frac{1}{2} + (p_j-1) \tau -p_j^{-1} \eps_{j}^a(\tau) \right) + O(\tau^2)$.
\end{theorem}

%\begin{corollary}
%If there is no data (unobserved mode), we have $\mu_y = \mathcal{N}\left(  0, C_\mu \right)$, 
%\[
%\check{\mu}_y = \mathcal{N} \left( \left[ C_\tau^{-1} + C^{-1} \mathcal{E}^a(\tau)  \right]^{-1} C^{-1} \mathcal{E}^b(\tau), \left[C_\tau^{-1} + C^{-1} \mathcal{E}^a(\tau)\right]^{-1}\right). 
%\]
%Using the assumptions of Theorem \ref{th:error-analysis}, the error for the $j$-th mode is given by
%\[
%\textup{D}_\textup{KL}\Big(\check{\mu}^{(j)}_y \,\Big|\Big|\, {\mu}^{(j)}_y \Big)  =  \lambda_j^{-2}\eps_{jj}^b(\tau)^2 \left(\frac{1}{2} + (p_j-1) \tau -p_j^{-1} \eps_{jj}^a(\tau) \right) + O(\tau^2). 
%\]    
%\end{corollary}

\begin{proof}
The proof relies on Proposition \ref{prop:stationary-langevin} and the fact that the $j$-th mode marginals $\check{\pi}^{(j)}_y$ and ${\pi}^{(j)}_y$ are Gaussian, $\mathcal{N}(\check{m}_j(\tau),\check{v}_j(\tau))$ and $\mathcal{N}(m_j,v_j)$, respectively, hence the Kullback-Leibler divergence has an explicit form and standard perturbation arguments lead to the desired estimates. Full details are provided in Appendix \ref{sec:thm:3-1}.
\end{proof}

\begin{remark} Note that Theorem \ref{thm:error-analysis} can provide a set of sufficient conditions that ensure that the global convergence error of the sampler %$\textup{D}_\textup{KL}\Big(\check{\mu}_y \,\Big|\Big|\, {\mu}_y \Big) = \sum_{j=1}^\infty \textup{D}_\textup{KL}\Big(\check{\mu}^{(j)}_y \,\Big|\Big|\, {\mu}^{(j)}_y \Big)$ 
is bounded in infinite dimensions:
$
\sum_j \left| \lambda_j^{-1} \eps_{j}^b(\tau)\right| <\infty$,
$\left|p_j^{-1} \eps_{j}^a(\tau)\right| <C_1$, 
and $\left|(A_N^\top y)_j \right|< C_2,
$ 
where $C_1,C_2$ do not depend on $j$.
\end{remark}

\section{The Essence of Preconditioning}\label{sec:preconditioning}
 %We assume that $C$  and $C_\mu$ are diagonal in the basis $(v_j)$, with eigenvalues $(\lambda_j)$ and $(\mu_j)$, respectively. 
We now elucidate the role of the preconditioner $C$ in the infinite-dimensional Gaussian setting introduced in the previous section. We begin with two preliminary remarks:
\vspace*{-0.6em}
\begin{itemize}
\item In our analysis, $C$ first appears in the forward diffusion \eqref{eq:forward-diffusion}, whose time-reversal learns the prior, and must be carried through the  Langevin sampler \eqref{eq:Langevin-SDE} to target the correct posterior. 
\item $C$ cannot be the identity: it must be trace-class to keep the diffusion well-posed and to stabilize the Langevin updates across all modes. Indeed, if $C = \text{Diag}(\lambda_j)$, the drift in the $j$-th mode contains the factor $\lambda_j [e^{-\tau} \mu_j + (1-e^{-\tau})\lambda_j]^{-1}$, which, unless $\lambda_j$ decays sufficiently fast, blows up like $\mu_j^{-1}$ as $j\to \infty$, making the sampler unstable at fine scales. This is a consequence of the infinite-dimensional setting, where $C_\mu$ must be trace-class.
%\item Proposition \ref{prop:stationary-langevin} shows that the choice of $C$ carries no improvement on the stationary distribution.   
\end{itemize}
\vspace*{-0.6em}
Since the preconditioner plays a role in the rate of convergence across all posterior modes, it is natural to ask whether there exists a $C$ that ensures a uniform convergence rate for the Langevin sampler. 
To this aim, in the next propositionwe  derive the mean reversion rate $\kappa$ of the preconditioned Langevin dynamics \eqref{eq:Langevin-SDE}; the proof is given in Appendix \ref{appendix:prop-4-1}.
\begin{proposition}\label{prop:reversion-rate}
Assume that $A_N^\top A_N$ is diagonal.  For the observed modes $j\leq N$, the mean reversion rate is
\begin{equation}
\kappa^{(j)} =  \lambda_j   \left(  \left[e^{-\tau} \mu_j + (1-e^{-\tau}) \lambda_j\right]^{-1}   +   \sigma^{-2} 
%\sum_{i=1}^N A_{ij}^2 
(A_N^\top A_N)_{jj}
+  \lambda_j^{-1} \eps^a_j(\tau)  \right),
\end{equation}
while for the unobserved modes $j>N$, 
$
\kappa^{(j)} =  \lambda_j  \big[ 
\left[e^{-\tau} \mu_j + (1-e^{-\tau}) \lambda_j\right]^{-1}  +    \lambda_j^{-1} \eps^a_j(\tau) \big]$.
%we have
%$$
%\kappa^{(j)} =  \lambda_j  \left[ 
%\left[e^{-\tau} \mu_j + (1-e^{-\tau}) \lambda_j\right]^{-1}  +    \lambda_j^{-1} %\eps^a_j(\tau) \right].
%%\operatorname*{Diag}\limits_{1\leq j \leq N}\left(\lambda_j\right)  \left[ \operatorname*{Diag}\limits_{1\leq j \leq N}\left( \left[e^{-\tau} \mu_j + (1-e^{-\tau}) \lambda_j\right]^{-1}\right) +  \operatorname*{Diag}\limits_{1\leq j \leq N}\left(\lambda_j^{-1} \eps^a_j(\tau)\right) \right].
%$$
\end{proposition}
We can make a few comments:
\vspace*{-0.6em}
\begin{itemize}
\item For the unobserved modes $j>N$, the convergence rate is $\lambda_j [e^{-\tau} \mu_j + (1-e^{-\tau})\lambda_j]^{-1}$ ($\simeq \lambda_j /\mu_j$ for small $\tau$) when the error $\eps_j^a$ is negligible, and therefore we should choose $\lambda_j  = \mu_j$ for all $j$ to get a convergence uniform in $j$, that is to say, $C=C_\mu$. 
\item For the observed modes $j\leq N$, the convergence rates for those modes such that $\mu_j \ll \sigma^2/
(A_N^\top A_N)_{jj}
$ (or $(A_N^\top A_N)_{jj}=0$) are $\lambda_j /\mu_j$, whereas for modes such that $\mu_j \gg \sigma^2 / 
%\sum_{i=1}^N A_{ij}^2
(A_N^\top A_N)_{jj}
$ the convergence rates are $\lambda_j \sigma^{-2}
(A_N^\top A_N)_{jj}
$. We should then choose $\lambda_j= [\mu_j^{-1}+\sigma^{-2} 
%\sum_{i=1}^N A_{ij}^2
(A_N^\top A_N)_{jj}
]^{-1}$, or equivalently $C = [C_{\mu}^{-1}+\sigma^{-2} A^\top A]^{-1}$.
\end{itemize}
\vspace*{-0.6em}
We now refine our analysis of the preconditioner by incorporating a first-order correction that accounts for the score approximation error at small $\tau$, the regime in which the score is typically learned. The proof of the following theorem relies on a straightforward perturbation argument; full details are given in Appendix \ref{sec:B-2}. 

%We now prove the existence of an optimal preconditioner and derive its explicit form. 

\begin{theorem}\label{thm:preconditioner}
In addition to Assumption \ref{assumption:score-error}, we further suppose that $A_N^\top A_N$ is diagonal, and that $ \eps_{j}^a(\tau)  = \eps_j^a \tau + O(\tau^2)$. Under these conditions, the optimal preconditioner $C$ is also diagonal in the basis $(v_j)$, with eigenvalues that admit the expansion 
$\lambda_j = \lambda_j^{(0)} + \lambda_j^{(1)} \tau   + O(\tau^2)$. 
For the observed modes $j\leq N$, we have 
\begin{equation}
\lambda_j^{(0)} = \left[ \mu^{-1}_{j} + \sigma^{-2}
%\sum_{i=1}^N A_{ij}^2 
(A_N^\top A_N)_{jj} 
\right]^{-1}, \qquad\lambda_j^{(1)} =   {\lambda_j^{(0)}}^{3}  \mu_j^{-2} - {\lambda_j^{(0)}}^2 {\mu}_j^{-1}  - \eps_j^a {\lambda_j^{(0)}}  .
\end{equation}
%\[
%\lambda_j^{(1)} =  \left(  \frac{ \left[ \mu_j^{-1} + \sigma^{-2} \sum_{i=1}^N A_{ij}^2  \right]^{-1/\alpha} \mu_j^{-1}  \left( {\mu}_j^{-1} \left[ \mu_j^{-1} + \sigma^{-2} \sum_{i=1}^N A_{ij}^2  \right]^{-1/\alpha}-1\right) + \eps_j^a }{\alpha}\right)\left[ \mu_j^{-1} + \sigma^{-2} \sum_{i=1}^N A_{ij}^2   \right]^{-1}.
%\]
For the unobserved modes $j>N$, we have
$\lambda_j^{(0)} = \mu_j$, $\lambda_j^{(1)}= -  \mu_j  \eps_j^a$. 
%In the basis $(v_j)$ The optimal preconditioner $C$ is given by 
%\begin{align*}
%C = C^{(0)} + C^{(1)} \tau   + O(\tau^2),
%\end{align*}
%where $C^{(0)} = \left[ C_\mu^{-1} + \sigma^{-2} A^\top A  \right]^{-1/\alpha}$,
%\[
%C^{(1)} =  \left(  \frac{ \left[ C_\mu^{-1} + \sigma^{-2} A^\top A  \right]^{-1/\alpha} C_\mu^{-1}  \left( {C_\mu}^{-1} \left[ C_\mu^{-1} + \sigma^{-2} A^\top A  \right]^{-1/\alpha}-I\right) - \mathcal{E}^a }{\alpha}\right)\left[ C_\mu^{-1} + \sigma^{-2} A^\top A  \right]^{-1} . 
%\]
\end{theorem}
%\begin{remark}
%For the unobserved modes, the optimal preconditioner $C$  results in
%\[
%C = C_\mu^{\frac{1}{\alpha}} + \alpha^{-1}\left[C_\mu^{\frac{2-\alpha}{\alpha}} - C_\mu^{\frac{1}{\alpha}} -\mathcal{E}^a C_\mu  \right] \tau  + O(\tau^2).  
%\]
%\end{remark}
%\begin{proof}
%\end{proof}
Based on Theorem  \ref{thm:preconditioner}, we  make a few comments: 
\begin{itemize} 
\item To compute the preconditioner $C$, one would need information on $A^\top A$, $\sigma$, $C_\mu$, and the score approximation error. Our analysis, however, suggests a simple and practical choice: take $C$ as close as possible to the prior covariance $C_\mu$. For higher-order modes, the leading-order term of the preconditioner coincides with $C_\mu$, and this approximation is particularly justified when the prior decays quickly, so that $\mu_j^{-1} \gg \sigma^{-2} (A^\top_N A_N)$ for the low-order modes. Any available knowledge of the posterior covariance or score error can then be used to refine this first approximation. 
\item This is not the first occurrence in the literature of an optimal preconditioner for diffusion models in infinite dimensions. For example, while analyzing the convergence error of time-reversed SDE dynamics in infinite dimensions, Pidstrigach et al. \cite{pidstrigach2024infinite} derived a similar result for the optimal $C$ by minimizing the Wasserstein-$2$ distance between the true data distribution and the learned sample distribution. Interestingly, assuming no data model and a perfect score, our framework yields the same optimal $C$, with a few caveats: in our case, the preconditioner arises directly from the mean-reversion rate of the Langevin dynamics. Hence, the optimal covariance we identify does not merely minimize an upper bound: it represents, under the stated assumptions, the best achievable choice in practice for ensuring a uniform rate of convergence across all modes. 
\end{itemize}

\section{Non-Gaussian Sampling}\label{sec:non-gaussian}

We can generalize the results of Sections \ref{sec:gaussian} and \ref{sec:preconditioning} by considering the case of a general class of prior measures  $\mu$ assumed to be absolutely continuous with respect to a Gaussian reference measure $\mathcal{N}(0,C_\mu)$ with density proportional to $\exp(-\Phi)$. We present the main ideas here, relegating the more quantitative results and proofs in Appendix \ref{sec:3:appendix}.

To reproduce the approach of the Gaussian setting, one first needs to diagonalize the Langevin SDE system, which in turn requires diagonalizing the score function $S(X,\tau; \, \mu)$.   

\begin{proposition}\label{prop:score-phi}
We assume that $C$ and $C_\mu$ have the same basis of eigenfunctions $(v_j)$ and we define $X^{(j)}= \langle X,v_j\rangle$. We assume that $\Phi(X)= \sum_j \phi_j(X^{(j)})$.
The score function \eqref{def:score-non-gaussian} can be written as $S(X, \tau;\,\mu) =  \sum_j S^{(j)}(X^{(j)}, \tau; \mu) v_j$,
where
\begin{align}
S^{(j)}(X^{(j)}, \tau; \, \mu)  =  - \lambda_j \partial_j\check{\phi}_j (X^{(j)},\tau)- 
%\frac{e^\tau p_j}{1+(e^\tau-1) p_j} 
s_j(\tau,\mu) X^{(j)},
\end{align}
with $\check{\phi}_j(X^{(j)},\tau) = - \log   \mathbb{E} \big[ \exp(-\phi_j(\widetilde{X}^{(j)}_0))\mid \widetilde{X}^{(j)}_\tau = X^{(j)} \big]$, and
\begin{equation}
\begin{pmatrix}
\widetilde{X}^{(j)}_0\\\widetilde{X}^{(j)}_\tau    
\end{pmatrix}
\sim 
\mathcal{N}
\left( 0, 
\begin{pmatrix}
\mu_j & e^{-\tau/2} \mu_j\\
e^{-\tau/2} \mu_j & e^{-\tau} \mu_j +(1-e^{-\tau}) \lambda_j 
\end{pmatrix}
\right).
\end{equation}
\end{proposition}
We assume a more general form of the score approximation error.

\begin{assumption}\label{assumption:score-error-phi}
 We consider an approximate score $S_\theta(X,\tau; \, \mu)$ such that 
\[
\big\langle S(X,\tau;\, \mu)-S_\theta(X,\tau;\, \mu), \; v_j\big\rangle =  \eps_j^a(\tau) \big[ X^{(j)}  + \partial_j \phi_j(X^{(j)}) \big]+ \eps^b_j(\tau).
\]     
\end{assumption}
With the learned score,  the preconditioned Langevin SDE for the observed modes $j\leq N$ becomes
\begin{align*}
dX_t^N  = - M_N X_t^N dt + b_N dt  +
%\sqrt{2}  \operatorname*{Diag}\limits_{1\leq j \leq N} \left(\lambda_j^{1/2}\right) 
\sqrt{2C_N} dW^N_t,
\end{align*}
where
\begin{align*}
M_N = &
 \operatorname*{Diag}\limits_{1\leq j \leq N}\left( 
 %\frac{\lambda_j}{e^{-\tau}\mu_j+(1-e^{-\tau}) \lambda_j}
 s_j(\tau;\mu)
 \right) + 
 %\operatorname*{Diag}\limits_{1\leq j \leq N}\left(\lambda_j  \right) 
 C_N 
 \sigma^{-2}A_N^\top A_N +  \operatorname*{Diag}\limits_{1\leq j \leq N}\left(\eps_j^a(\tau)\right) ,\\
b_N =&    %\operatorname*{Diag}\limits_{1\leq j \leq N}\left(\lambda_j\right) 
C_N \sigma^{-2} A_N^\top y - 
C_N \operatorname*{Diag}\limits_{1\leq j \leq N}\left(  \partial_j \check{\phi}_j\right)  -   \operatorname*{Diag}\limits_{1\leq j \leq N}\left(\eps^a_j(\tau) \partial_j \phi_j (X^{(j)}_t)\right) - \operatorname*{Diag}\limits_{1\leq j \leq N}\left(\eps_j^b(\tau)\right).
\end{align*}
The SDE for the unobserved modes $j>N$ can be obtained by taking $A_N=0$  above.

The stationary distribution of the preconditioned Langevin SDE is derived in the following proposition, which makes explicit the dependence on the score approximation error.
\begin{proposition}\label{prop:stationary-phi}
Let Assumption \ref{assumption:score-error-phi} hold true. Under the hypotheses of the previous proposition,  the preconditioned Langevin with approximate score in the drift term has $\check{\pi}_y $ as its stationary distribution. It is absolutely continuous with respect to $\mathcal{N} (\check{m} (\tau), \check{v} (\tau))$, and is given by the density
\begin{equation}
\frac{\text{d}\check{\pi}_y}{\text{d} \mathcal{N} (\check{m} (\tau), \check{v} (\tau))}(X,\tau) \propto \exp\left( -\check{\Phi}  (X,\tau )\right)  .
\end{equation}
For the observed modes $j\leq N$, the negative log-density is
$\check{\Phi}^N ( X^N,\tau) = \sum_{j=1}^N \big[ \check{\phi}_j (X^{(j)},\tau) + \lambda_j^{-1} \eps^a_j(\tau) \phi_j(X_t^{(j)}) \big]$,
%\begin{align}
%\check{\Phi}^N ( X^N,\tau) =& \sum_{j=1}^N \left[ \check{\phi}_j (X^{(j)},\tau) + \lambda_j^{-1} \eps^a_j(\tau) \phi_j(X_t^{(j)}) \right],
%%\\
%%\check{v}^N(\tau) =& \left[ C_N^{-1} \operatorname*{Diag}\limits_{1\leq j \leq N}\left(
%%%\frac{1}{e^{-\tau}\mu_j + (1-e^{-\tau})\lambda_j}
%%s_j(\tau;\mu)
%%\right) + \sigma^{-2} A_N^\top A_N +  C_N^{-1} \operatorname*{Diag}\limits_{1\leq j \leq %N}\left(  \eps_j^a(\tau) \right) \right]^{-1},\\
%%\check{m}^N (\tau) = & \check{v}^N(\tau) \left[\sigma^{-2}A_N^\top y - C_N^{-1} \operatorname*{Diag}\limits_{1\leq j \leq N}\left(  \eps_j^b(\tau)\right) \right],
%\end{align}
the covariance $\check{v}^N(\tau)$ and mean $\check{m}^N (\tau)$ are given by (\ref{expres:vn}-\ref{expres:mn}).
For the unobserved modes $j>N$, the negative log-density is $\check{\Phi}^{(j)} (X^{(j)},\tau) =  \check{\phi}_j (X^{(j)},\tau) + \lambda_j^{-1} \eps^a_j(\tau) \phi_j(X_t^{(j)})$,
the covariance $\check{v}^{(j)}(\tau)$ and mean $\check{m}^{(j)}(\tau)$ are given by (\ref{expres:vjmj}).
%\begin{align}
%%\check{\Phi}^{(j)} (X^{(j)},\tau) =&  \check{\phi}_j (X^{(j)},\tau) + \lambda_j^{-1} \eps^a_j(\tau) \phi_j(X_t^{(j)}),\\
%\check{v}^{(j)}(\tau) = & \left[ 
%%\frac{1}{e^{-\tau}\mu_j + (1-e^{-\tau})\lambda_j}  
% \lambda_j^{-1} s_j(\tau;\mu)
%+  \lambda_j^{-1} \eps_j^a(\tau)  \right]^{-1}, \qquad \check{m}^{(j)} (\tau) =  -\check{v}^{(j)}(\tau)  \lambda_j^{-1} \eps_j^b(\tau).
%\end{align}
\end{proposition}
The interested reader can find in Appendix \ref{sec:3:appendix} a quantitative analysis of this general case, including Theorem \ref{thm:non-gaussian}, which provides an error analysis analogous to Theorem \ref{thm:error-analysis} for the Gaussian case. 
Here we make a few  qualitative comments:%, similar to those in the Gaussian setting, with the main difference being the assumptions required on the potential $\Phi$ for determining the optimal preconditioner:
\vspace*{-0.6em}
\begin{itemize}
\item If $\eps_j^a=\eps_j^b=0$ for all $j$, then
$\check{\Phi}(X,\tau)  \stackrel{\tau \to 0}{\to} \sum_j\phi_j (X^{(j)}) $,
%\begin{align*}
%& \check{\Phi}(X,\tau) = \sum_{j} - \log   \mathbb{E} \left[ \exp(-\phi_j(\widetilde{X}^{(j)}_0))\mid \widetilde{X}^{(j)}_\tau = X^{(j)} \right] \stackrel{\tau \to 0}{\to} \sum_j\phi_j (X^{(j)}), 
%\\
%& \check{v} (\tau) 
%\stackrel{\tau \to 0}{\to}
%\left[ C_\mu^{-1} + \sigma^{-2}A^\top A \right]^{-1}, \qquad 
%\check{m} (\tau)  
%\stackrel{\tau \to 0}{\to}
%\left[ C_\mu^{-1}+\sigma^{-2}A^\top A \right]^{-1} \sigma^{-2} A^\top y,
%\end{align*}
$\check{v} (\tau) 
\stackrel{\tau \to 0}{\to}
\left[ C_\mu^{-1} + \sigma^{-2}A^\top A \right]^{-1}$, 
$\check{m} (\tau)  
\stackrel{\tau \to 0}{\to}
\left[ C_\mu^{-1}+\sigma^{-2}A^\top A \right]^{-1} \sigma^{-2} A^\top y,$
that is to say, we get the posterior given the data.
\item The preconditioner influences the convergence rate. Consider the case in which $\phi_j$ is convex, i.e., $\partial_j^2 \phi_j \geq C_{\phi_j} >0$. In this case, for the unobserved modes $j>N$, the convergence rate of the $j$-th mode is $\simeq \lambda_j[\mu_j^{-1} + C_{\phi_j}]$ for small $\tau$, assuming the error $\eps_j^a$ is negligible. To achieve a convergence rate that is uniform in $j$, we should then choose $\lambda_j = [\mu_j^{-1} + C_{\phi_j} ]^{-1}$.
For the observed modes $j\leq N$, and assuming $A_N^\top A_N$ is diagonal for simplicity, we should instead choose $\lambda_j = [ \mu_j^{-1} + \sigma^{-2} 
%\sum_{i=1}^N A_{ij}^2 
(A_N^\top A_N)_{jj} + C_{\phi_j} ]^{-1}$.
\item If the error $|\eps_j^a | \ll \lambda_j$, its impact on the stationary distribution of $X^{(j)}_t$ is small. 
%\item 
Like the Gaussian case, the error $\eps_j^b$ can induce a bias. It can be large since $\lambda_j^{-1} %\stackrel{j \to \infty}{\to} 
\to
+\infty$ as $j\to +\infty$.
\end{itemize}

\section{Illustrations}\label{sec:illustrations}

%We aim to illustrate our theory by showing (i) discretization-invariance of the sampler (ii) effects of the score approximation error in the Gaussian setting (iii) the role of preconditioning in ensuring uniform rate of convergence.

We verify our theory by applying the preconditioned Langevin dynamics with SGM \eqref{eq:Langevin-SDE} to two linear inverse problems: one based on the Karhunen-Lo\`eve (KL) expansion of the Brownian sheet \cite{wang2008karhunen}, and the other one on an inverse source problem for the heat equation \cite{stuart2010inverse}. Both examples are consistent with the theory of the paper. Implementation and further details are provided in Appendix \ref{sec:4:appendix}.

{\bf Brownian sheet. \hspace{0.5mm}}  We illustrate the discretization-invariance property of our approach on the Brownian sheet, represented by its truncated KL expansion $B^N(x) = \sum_{j,k=1}^N \phi_{j,k} (x) \eta_{j,k}$, $x\in [0,1]^2$, where $\eta_{j,k} \sim \mathcal{N}(0,\mu_{j,k})$ and $(\phi_{j,k}, \mu_{j,k})$ are the KL eigenpairs. For $M \leq N$, the inverse problem consists of reconstructing the KL coefficients and Brownian sheet from noisy data $y_{j,k} = \tilde{\eta}_{j,k} + \eps_{j,k}$, $j,k\leq M$, with $\tilde{\eta}_{j,k} \sim \mathcal{N}(0,\mu_{j,k})$, $\eps_{j,k} \sim \mathcal{N}(0,0.01^2)$ (i.e. $A_{jk}=1$ if $j=k\leq M$ and $0$ otherwise).\,Figure \ref{fig:brownian-sheet} shows the  robustness of our approach with respect to the number of modes$\,M^2$.  
%{\bf Brownian sheet. \hspace{0.5mm}}  We illustrate the discretization-invariance property of our approach on the Brownian sheet, represented by its truncated KL expansion $B^N(x_1,x_2) = \sum_{j,k=1}^N \sqrt{\mu_{j,k}} \phi_{j,k} (x_1,x_2) \eta_{j,k}$, where $(\phi_{j,k}, \mu_{j,k})$ are the eigenpairs  and $\eta_{j,k} \sim \mathcal{N}(0,1)$. For $j,k \leq M \leq N$, the inverse problem consists of reconstructing the KL coefficients $\tilde{\eta}_{j,k}$ from noisy data $y_{j,k} = \tilde{\eta}_{j,k} + \eps_{j,k}$, with $\tilde{\eta}_{j,k} \sim \mathcal{N}(0,\mu_{j,k})$, $\eps_{j,k} \sim \mathcal{N}(0,0.01^2)$, and forward map $A=I$. In Figure \ref{fig:brownian-sheet}, by increasing  $M$, we show the  robustness of the infinite-dimensional approach with respect to the number of modes.  

\begin{figure}[h]
    \centering
    \includegraphics[width=\textwidth]{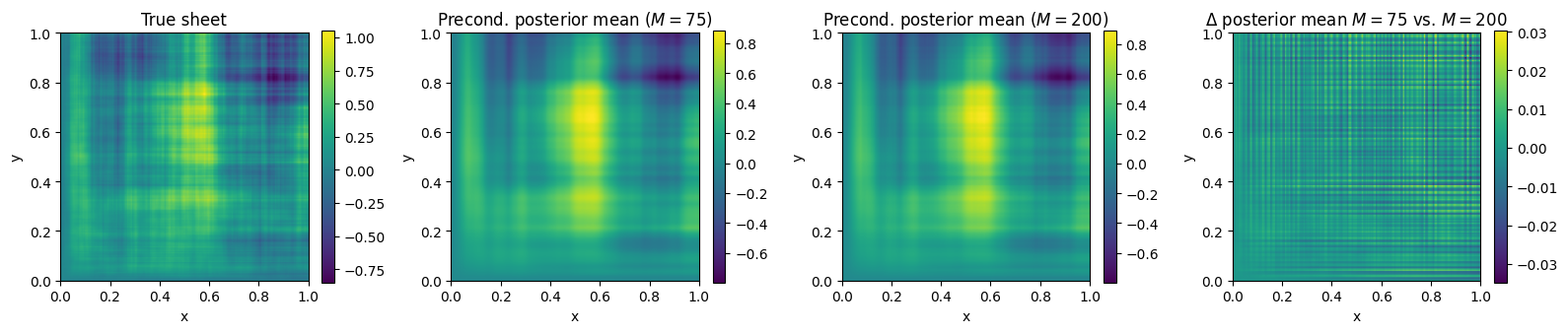}
    \caption{Discretization invariance in reconstructing the KL expansion of the Brownian sheet for increasing number of observed modes $M^2$, with $M=75$ and $M=200$. Here $N=200$.}
    \label{fig:brownian-sheet}
\end{figure}

{\bf Heat equation. \hspace{0.5mm}}  We verify the benefits of the optimal preconditioner $C$ from Theorem \ref{thm:preconditioner} by considering the ill-posed inverse problem of recovering the initial condition $u(x,0)$, $x \in [0,1]^2$, of the heat equation from noisy observations of the solution $u(x,T)$ at time $T=0.1$. Expanding in the eigenpairs $(\psi_{j,k}, \zeta_{j,k})$ of the Dirichlet Laplacian,  one finds that $u(x,t) =\sum_{j,k} e^{-\zeta_{j,k}t}g_{j,k} \psi_{j,k}(x)$, where $g_{j,k} =\langle u(\cdot,0),\psi_{j,k}\rangle $. The inverse problem diagonalizes: we observe $y_{j,k} = e^{-\zeta_{j,k} T} g_{j,k} + \eps_{j,k}$, $j,k\leq M$, with  $g_{j,k}\sim \mathcal{N}(0, e^{- \beta \zeta_{j,k}} )$, $\beta=0.1$, $\eps_{j,k} \sim \mathcal{N}(0, 0.005^2)$. In Figure \ref{fig:heat-ip} %(where $M=15$), 
we compare reconstructions using Langevin dynamics preconditioned with the optimal $C$  ($3$rd column) and vanilla Langevin ($4$th column). Both samplers use a score perturbed by a relative error $\eps_j^a \sim {\cal N}(0, 0.1^2)$ scaled by a small $\tau$, as assumed in Theorem \ref{thm:preconditioner}. The results support our theory: (i) the preconditioned sampler is robust to score approximation error, as expected from the design of $C$ (Theorem \ref{thm:preconditioner}); ii) as shown in the autocorrelation plot in Figure \ref{fig:ACF} (corresponding to the top row of Figure \ref{fig:heat-ip}), the modes converge faster and more uniformly than with the vanilla dynamics, since $C$ targets the optimal mean reversion rates (Proposition \ref{prop:reversion-rate}) ; iii) vanilla Langevin deteriorates as the score error increases (Figure \ref{fig:heat-ip})
%, top row: $\tau=10^{-3}$; bottom row: $\tau=10^{-1}$), 
due to amplification at fine scales, reproducing the pathological behavior seen in  Figure~\ref{fig:langevin_challenges}.% and illustrating the need for the ``\emph{apply-algorithm-then-discretize}'' approach intrinsic to the infinite-dimensional framework. 

\begin{figure}[h]
    \centering
    \includegraphics[width=\textwidth]{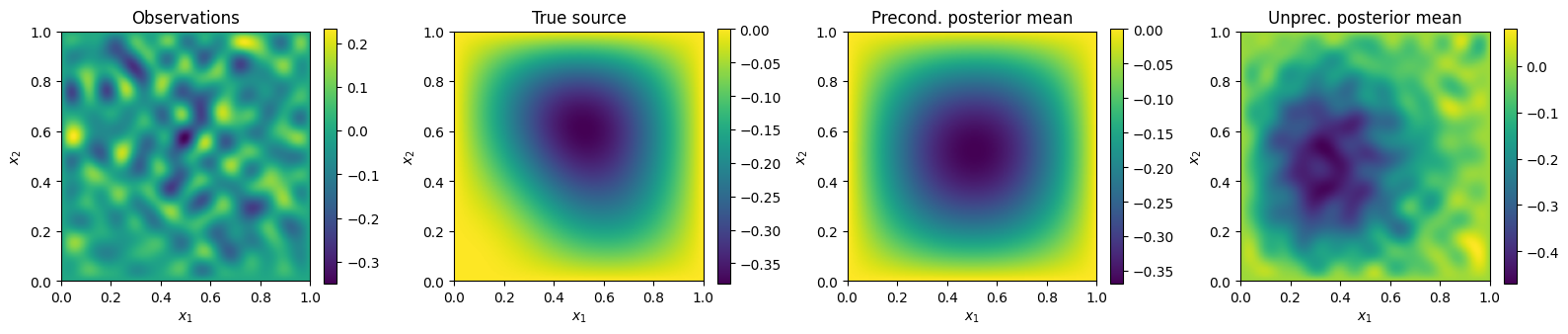}
    \includegraphics[width=\textwidth]{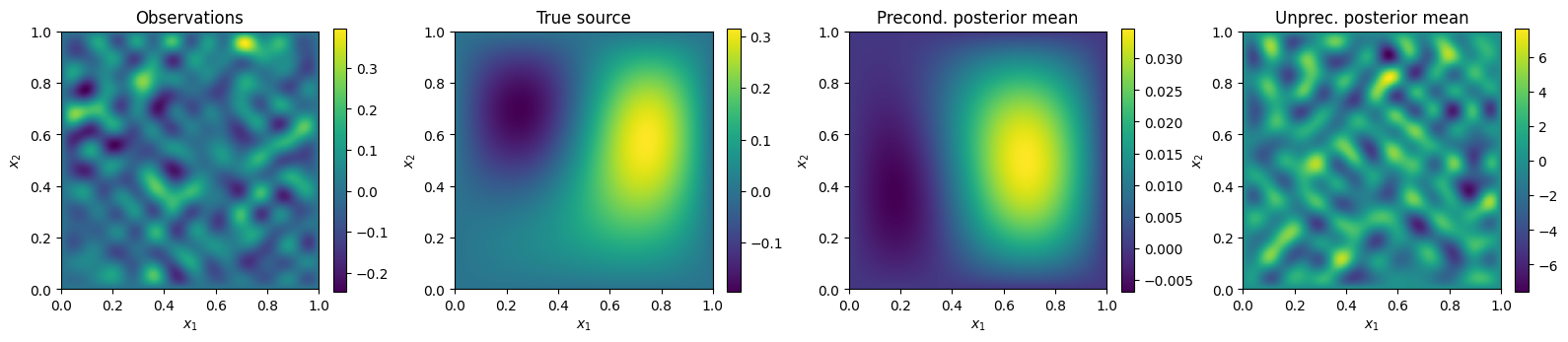}
\caption{Effects of preconditioning in Langevin sampling for the inverse heat source problem.\,$M=15$; top row: $\tau=10^{-3}$, bottom row: $\tau=10^{-1}$. The $4$th column shows the issues of Figure \ref{fig:langevin_challenges}.
%, using an imperfect score with relative error scaled by $\tau=10^{-3}$ (top) and $\tau=10^{-1}$ (bottom). 
}
    \label{fig:heat-ip}
\end{figure}

\begin{figure}[h]
    \centering
    \begin{subfigure}{0.45\linewidth}
        \centering
        \includegraphics[width=\linewidth]{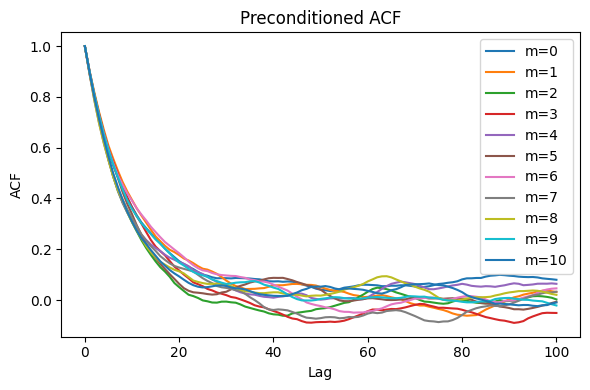}
%        \caption{Case $\tau =10^{-3}$, Preconditioned Langevin, inverse source problem}
%        \label{fig:subfig1}
    \end{subfigure}
    \hfill
    \begin{subfigure}{0.45\linewidth}
        \centering
        \includegraphics[width=\linewidth]{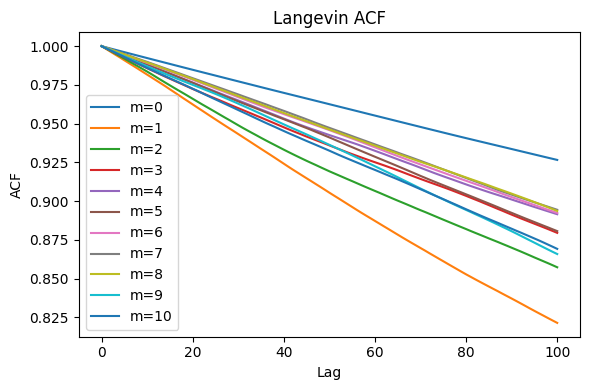}
 %       \caption{Case $\tau =10^{-3}$, Unpreconditioned Langevin, inverse source problem}
  %      \label{fig:subfig2}
    \end{subfigure}
    \caption{Mode autocorrelation. Left: Preconditioned Langevin with SGM. Right: Vanilla Langevin.}
    \label{fig:ACF}
\end{figure}

\section{Discussion and Future Work}
We studied a popular sampler---a Langevin-type diffusion driven by score-based generative priors---directly in the infinite-dimensional Bayesian setting, rather than in the usual finite-dimensional one. We showed that na\"{i}vely applying standard techniques in infinite dimensions leads to several issues. To ensure provable posterior sampling and discretization-invariance, our analysis shows that preconditioning the vanilla Langevin is necessary. We prove detailed convergence error estimates and the existence (and form) of an optimal preconditioner---depending on both the forward map$\,A$ and the score error---that yields uniform convergence across all modes. 

As is standard in infinite-dimensional analysis, our results rely on some simplifying assumptions: finite-dimensional data, and co-diagonalizability of the prior and the diffusion's noise covariance. In some parts, we also assumed that $A^\top A$ is diagonal, a common assumption in the theory of  linear Bayesian inverse problems \cite{stuart2010inverse, knapik2011bayesian}. This is not merely a technical convenience: in many classical linear inverse problems, such as the heat equation, tomography, or inverse scattering for Schr\"odinger-type operators under the Born approximation, the forward operator $A$ is compact, and hence one can always find a basis in which
$A^\top A$ is diagonal. 

Nevertheless, the main conclusions of our analysis remain valid even without the diagonalization assumption. For example, the asymptotic expansion in Eq. \eqref{eq:thm-3-1} of Theorem \ref{thm:error-analysis} can be extended to non-diagonal by  replacing scalar expansions with their corresponding matrix-series counterparts, while the arguments for the higher-order modes remains the same. Likewise, in Section \ref{sec:preconditioning}, one can verify that under a perfect score function the optimal preconditioner still takes the form $C=[C_\mu^{-1} + \sigma^{-2} A^\top A]^{-1}$.

Several open questions remain. In particular, how do these results extend to nonlinear inverse problems? Extending the analysis of \cite{baldassari2024taming} to determine an optimal preconditioner for nonlinear inverse problems represents an important direction for future work.

\newpage 

\section*{Acknowledgments}
JG was supported by Agence de l’Innovation de Défense – AID - via Centre Interdisciplinaire d’Etudes pour la Défense et la Sécurité – CIEDS - (project 2021 - PRODIPO). KS was supported by Air Force Office of Scientific Research under grant FA9550-22-1-0176 and the National Science Foundation under grant DMS-2308389. MVdH acknowledges support from the Department of Energy under grant DE-SC0020345, Oxy, the corporate members of the Geo-Mathematical Imaging Group at Rice University, and the Simons Foundation under the MATH + X program. 
%\section*{References}
\bibliography{arxiv_refs.bib}

\begin{thebibliography}{50}
\providecommand{\natexlab}[1]{#1}
\providecommand{\url}[1]{\texttt{#1}}
\expandafter\ifx\csname urlstyle\endcsname\relax
  \providecommand{\doi}[1]{doi: #1}\else
  \providecommand{\doi}{doi: \begingroup \urlstyle{rm}\Url}\fi

\bibitem[Tarantola(2005)]{tarantola2005inverse}
A. Tarantola.
\newblock \emph{Inverse problem theory and methods for model parameter estimation}.
\newblock SIAM, 2005.

\bibitem[Hadamard(1923)]{hadamard2003lectures}
J. Hadamard.
\newblock \emph{Lectures on Cauchy's problem in linear partial differential equations}.
\newblock Yale University Press, 1923.

\bibitem[Stuart(2010)]{stuart2010inverse}
A.~M. Stuart.
\newblock Inverse problems: a bayesian perspective.
\newblock \emph{Acta numerica}, 19:\penalty0 451--559, 2010.

\bibitem[Knapik et~al.(2011)Knapik, van~der Vaart, and van Zanten]{knapik2011bayesian}
B.~T. Knapik, A.~W. van~der Vaart, and J.~H. van Zanten.
\newblock Bayesian inverse problems with {G}aussian priors.
\newblock \emph{The Annals of Statistics}, 39\penalty0 (5):\penalty0 2626--2657, 2011.

\bibitem[Dashti and Stuart(2011)]{dashti2011uncertainty}
M. Dashti and A.~M. Stuart.
\newblock Uncertainty quantification and weak approximation of an elliptic inverse problem.
\newblock \emph{SIAM Journal on Numerical Analysis}, 49\penalty0 (6):\penalty0 2524--2542, 2011.

\bibitem[Stuart(2014)]{stuart2014uncertainty}
A.~M. Stuart.
\newblock Uncertainty quantification in bayesian inversion.
\newblock \emph{ICM2014. Invited Lecture}, 1279, 2014.

\bibitem[Dashti and Stuart(2013)]{dashti2013bayesian}
M. Dashti and A.~M. Stuart.
\newblock The bayesian approach to inverse problems.
\newblock \emph{arXiv preprint arXiv:1302.6989}, 2013.

\bibitem[Cotter et~al.(2013)Cotter, Roberts, Stuart, and White]{cotter2013mcmc}
S.~L. Cotter, G.~O. Roberts, A.~M. Stuart, and D. White.
\newblock {MCMC} methods for functions: {M}odifying old algorithms to make them faster.
\newblock \emph{Statistical Science}, 28\penalty0 (3):\penalty0 424--446, 2013.

\bibitem[Lassas and Siltanen(2004)]{lassas2004can}
M. Lassas and S. Siltanen.
\newblock Can one use total variation prior for edge-preserving bayesian inversion?
\newblock \emph{Inverse problems}, 20\penalty0 (5):\penalty0 1537, 2004.

\bibitem[Liu et~al.(2016)Liu, Lee, and Jordan]{liu2016kernelized}
Q. Liu, J. Lee, and M. Jordan.
\newblock A kernelized stein discrepancy for goodness-of-fit tests.
\newblock In \emph{International conference on machine learning}, pages 276--284. PMLR, 2016.

\bibitem[Vincent(2011)]{vincent2011connection}
P. Vincent.
\newblock A connection between score matching and denoising autoencoders.
\newblock \emph{Neural computation}, 23\penalty0 (7):\penalty0 1661--1674, 2011.

\bibitem[Song and Ermon(2020)]{song2020improved}
Y. Song and S. Ermon.
\newblock Improved techniques for training score-based generative models.
\newblock \emph{Advances in neural information processing systems}, 33:\penalty0 12438--12448, 2020.

\bibitem[Song and Ermon(2019)]{song2019generative}
Y. Song and S. Ermon.
\newblock Generative modeling by estimating gradients of the data distribution.
\newblock \emph{Advances in neural information processing systems}, 32, 2019.

\bibitem[Song et~al.(2020)Song, Sohl-Dickstein, Kingma, Kumar, Ermon, and Poole]{song2020score}
Y. Song, J. Sohl-Dickstein, D.~P. Kingma, A. Kumar, S. Ermon, and B. Poole.
\newblock Score-based generative modeling through stochastic differential equations.
\newblock \emph{arXiv preprint arXiv:2011.13456}, 2020.

\bibitem[Sohl-Dickstein et~al.(2015)Sohl-Dickstein, Weiss, Maheswaranathan, and Ganguli]{sohl2015deep}
J. Sohl-Dickstein, E. Weiss, N. Maheswaranathan, and S. Ganguli.
\newblock Deep unsupervised learning using nonequilibrium thermodynamics.
\newblock In \emph{International conference on machine learning}, pages 2256--2265. PMLR, 2015.

\bibitem[Ho et~al.(2020)Ho, Jain, and Abbeel]{ho2020denoising}
J. Ho, A. Jain, and P. Abbeel.
\newblock Denoising diffusion probabilistic models.
\newblock \emph{Advances in neural information processing systems}, 33:\penalty0 6840--6851, 2020.

\bibitem[Batzolis et~al.(2021)Batzolis, Stanczuk, Sch{\"o}nlieb, and Etmann]{batzolis2021conditional}
G. Batzolis, J. Stanczuk, C.-B. Sch{\"o}nlieb, and C. Etmann.
\newblock Conditional image generation with score-based diffusion models.
\newblock \emph{arXiv preprint arXiv:2111.13606}, 2021.

\bibitem[Kawar et~al.(2021)Kawar, Vaksman, and Elad]{kawar2021snips}
B. Kawar, G. Vaksman, and M. Elad.
\newblock Snips: Solving noisy inverse problems stochastically.
\newblock \emph{Advances in Neural Information Processing Systems}, 34:\penalty0 21757--21769, 2021.

\bibitem[Jalal et~al.(2021)Jalal, Arvinte, Daras, Price, Dimakis, and Tamir]{jalal2021robust}
A. Jalal, M. Arvinte, G. Daras, E. Price, A.~G. Dimakis, and J. Tamir.
\newblock Robust compressed sensing mri with deep generative priors.
\newblock \emph{Advances in Neural Information Processing Systems}, 34:\penalty0 14938--14954, 2021.

\bibitem[Baldassari et~al.(2024{\natexlab{a}})Baldassari, Siahkoohi, Garnier, Solna, and de~Hoop]{baldassari2024conditional}
L. Baldassari, A. Siahkoohi, J. Garnier, K. Solna, and M.~V. de~Hoop.
\newblock Conditional score-based diffusion models for bayesian inference in infinite dimensions.
\newblock \emph{Advances in Neural Information Processing Systems}, 36, 2024{\natexlab{a}}.

\bibitem[Baldassari et~al.(2024{\natexlab{b}})Baldassari, Siahkoohi, Garnier, Solna, and de~Hoop]{baldassari2024taming}
L. Baldassari, A. Siahkoohi, J. Garnier, K. Solna, and M.~V. de~Hoop.
\newblock Taming score-based diffusion priors for infinite-dimensional nonlinear inverse problems.
\newblock \emph{arXiv preprint arXiv:2405.15676}, 2024{\natexlab{b}}.

\bibitem[Hosseini et~al.(2023)Hosseini, Hsu, and Taghvaei]{hosseini2023conditional}
B. Hosseini, A.~W. Hsu, and A. Taghvaei.
\newblock Conditional optimal transport on function spaces.
\newblock \emph{arXiv preprint arXiv:2311.05672}, 2023.

\bibitem[Hagemann et~al.(2023)Hagemann, Mildenberger, Ruthotto, Steidl, and Yang]{hagemann2023multilevel}
P. Hagemann, S. Mildenberger, L. Ruthotto, G. Steidl, and N.~T. Yang.
\newblock Multilevel diffusion: Infinite dimensional score-based diffusion models for image generation.
\newblock \emph{arXiv preprint arXiv:2303.04772}, 2023.

\bibitem[Feng et~al.(2023)Feng, Smith, Rubinstein, Chang, Bouman, and Freeman]{feng2023score}
B.~T. Feng, J. Smith, M. Rubinstein, H. Chang, K.~L. Bouman, and W.~T. Freeman.
\newblock Score-based diffusion models as principled priors for inverse imaging.
\newblock In \emph{Proceedings of the IEEE/CVF International Conference on Computer Vision}, pages 10520--10531, 2023.

\bibitem[Sun et~al.(2024)Sun, Wu, Chen, Feng, and Bouman]{sun2023provable}
Y. Sun, Z. Wu, Y. Chen, B.~T. Feng, and K.~L. Bouman.
\newblock Provable probabilistic imaging using score-based generative priors.
\newblock \emph{IEEE Transactions on Computational Imaging}, 2024.

\bibitem[Xu and Chi(2024)]{xu2024provably}
X. Xu and Y. Chi.
\newblock Provably robust score-based diffusion posterior sampling for plug-and-play image reconstruction.
\newblock \emph{arXiv preprint arXiv:2403.17042}, 2024.

\bibitem[Stuart et~al.(2004)Stuart, Voss, and Wilberg]{stuart2004conditional}
A.~M. Stuart, J. Voss, and P. Wilberg.
\newblock Conditional path sampling of sdes and the langevin mcmc method.
\newblock 2004.

\bibitem[Hairer et~al.(2005)Hairer, Stuart, Voss, and Wiberg]{hairer2005analysis}
M. Hairer, A.~M. Stuart, J. Voss, and P. Wiberg.
\newblock Analysis of spdes arising in path sampling. part i: The gaussian case.
\newblock 2005.

\bibitem[Hairer et~al.(2007)Hairer, Stuart, and Voss]{hairer2007analysis}
M. Hairer, A.~M. Stuart, and J. Voss.
\newblock Analysis of spdes arising in path sampling part ii: The nonlinear case.
\newblock 2007.

\bibitem[Beskos et~al.(2017)Beskos, Girolami, Lan, Farrell, and Stuart]{beskos2017geometric}
A. Beskos, M. Girolami, S. Lan, P.~E. Farrell, and A.~M. Stuart.
\newblock Geometric mcmc for infinite-dimensional inverse problems.
\newblock \emph{Journal of Computational Physics}, 335:\penalty0 327--351, 2017.

\bibitem[Wallin and Vadlamani(2018)]{wallin2018infinite}
J. Wallin and S. Vadlamani.
\newblock Infinite dimensional adaptive mcmc for gaussian processes.
\newblock \emph{arXiv preprint arXiv:1804.04859}, 2018.

\bibitem[Durmus and Moulines(2019)]{durmus2019high}
A. Durmus and {\'E}. Moulines.
\newblock {High-dimensional {B}ayesian inference via the unadjusted {L}angevin algorithm}.
\newblock \emph{Bernoulli}, 25\penalty0 (4A):\penalty0 2854--2882, 2019.

\bibitem[Durmus and Moulines(2017)]{durmus2017nonasymptotic}
A. Durmus and {\'E}. Moulines.
\newblock Nonasymptotic convergence analysis for the unadjusted {L}angevin algorithm.
\newblock \emph{The Annals of Applied Probability}, 27\penalty0 (3):\penalty0 1551--1587, 2017.

\bibitem[Dalalyan(2017)]{dalalyan2017theoretical}
A.~S. Dalalyan.
\newblock Theoretical guarantees for approximate sampling from smooth and log-concave densities.
\newblock \emph{Journal of the Royal Statistical Society Series B: Statistical Methodology}, 79\penalty0 (3):\penalty0 651--676, 2017.

\bibitem[Hairer et~al.(2014)Hairer, Stuart, and Vollmer]{hairer2014spectral}
M. Hairer, A.~M. Stuart, and S.~J. Vollmer.
\newblock Spectral gaps for a {M}etropolis--{H}astings algorithm in infinite dimensions.
\newblock \emph{The Annals of Applied Probability}, 24\penalty0 (6):\penalty0 2455--2490, 2014.

\bibitem[Cui et~al.(2016)Cui, Law, and Marzouk]{cui2016dimension}
T. Cui, K.~J. Law, and Y.~M. Marzouk.
\newblock Dimension-independent likelihood-informed mcmc.
\newblock \emph{Journal of Computational Physics}, 304:\penalty0 109--137, 2016.

\bibitem[Cui et~al.(2024)Cui, Detommaso, and Scheichl]{cui2024multilevel}
T. Cui, G. Detommaso, and R. Scheichl.
\newblock Multilevel dimension-independent likelihood-informed mcmc for large-scale inverse problems.
\newblock \emph{Inverse Problems}, 40\penalty0 (3):\penalty0 035005, 2024.

\bibitem[Beskos et~al.(2018)Beskos, Jasra, Law, Marzouk, and Zhou]{beskos2018multilevel}
A. Beskos, A. Jasra, K. Law, Y. Marzouk, and Y. Zhou.
\newblock Multilevel sequential monte carlo with dimension-independent likelihood-informed proposals.
\newblock \emph{SIAM/ASA Journal on Uncertainty Quantification}, 6\penalty0 (2):\penalty0 762--786, 2018.

\bibitem[Morzfeld et~al.(2019)Morzfeld, Tong, and Marzouk]{morzfeld2019localization}
M. Morzfeld, X.~T. Tong, and Y.~M. Marzouk.
\newblock Localization for mcmc: sampling high-dimensional posterior distributions with local structure.
\newblock \emph{Journal of Computational Physics}, 380:\penalty0 1--28, 2019.

\bibitem[Beskos et~al.(2008)Beskos, Roberts, Stuart, and Voss]{beskos2008mcmc}
A. Beskos, G. Roberts, A. Stuart, and J. Voss.
\newblock Mcmc methods for diffusion bridges.
\newblock \emph{Stochastics and Dynamics}, 8\penalty0 (03):\penalty0 319--350, 2008.

\bibitem[Ma et~al.(2022)Ma, Zhang, Zhu, and Feng]{ma2022accelerating}
H. Ma, L. Zhang, X. Zhu, and J. Feng.
\newblock Accelerating score-based generative models with preconditioned diffusion sampling.
\newblock In \emph{European conference on computer vision}, pages 1--16. Springer, 2022.

\bibitem[Ma et~al.(2025)Ma, Zhu, Feng, and Zhang]{ma2025preconditioned}
H. Ma, X. Zhu, J. Feng, and L. Zhang.
\newblock Preconditioned score-based generative models.
\newblock \emph{International Journal of Computer Vision}, pages 1--27, 2025.

\bibitem[Franzese et~al.(2024)Franzese, Corallo, Rossi, Heinonen, Filippone, and Michiardi]{franzese2024continuous}
G. Franzese, G. Corallo, S. Rossi, M. Heinonen, M. Filippone, and P. Michiardi.
\newblock Continuous-time functional diffusion processes.
\newblock \emph{Advances in Neural Information Processing Systems}, 36, 2024.

\bibitem[Franzese and Michiardi(2025)]{franzese2025generative}
G. Franzese and P. Michiardi.
\newblock Generative diffusion models in infinite dimensions: a survey.
\newblock \emph{Philosophical Transactions A}, 383\penalty0 (2299):\penalty0 20240322, 2025.

\bibitem[Lim et~al.(2023)Lim, Kovachki, Baptista, Beckham, Azizzadenesheli, Kossaifi, Voleti, Song, Kreis, Kautz, et~al.]{lim2023score}
J.~H. Lim, N.~B. Kovachki, R. Baptista, C. Beckham, K. Azizzadenesheli, J. Kossaifi, V. Voleti, J. Song, K. Kreis, J. Kautz, et~al.
\newblock Score-based diffusion models in function space.
\newblock \emph{arXiv preprint arXiv:2302.07400}, 2023.

\bibitem[Kerrigan et~al.(2022)Kerrigan, Ley, and Smyth]{kerrigan2022diffusion}
G. Kerrigan, J. Ley, and P. Smyth.
\newblock Diffusion generative models in infinite dimensions.
\newblock \emph{arXiv preprint arXiv:2212.00886}, 2022.

\bibitem[Lim et~al.(2024)Lim, YOON, Byun, Kang, Kim, Lee, and Choi]{lim2024score}
S. Lim, E.~B. YOON, T. Byun, T. Kang, S. Kim, K. Lee, and S. Choi.
\newblock Score-based generative modeling through stochastic evolution equations in hilbert spaces.
\newblock \emph{Advances in Neural Information Processing Systems}, 36, 2024.

\bibitem[Bond-Taylor and Willcocks(2023)]{bond2023infty}
S. Bond-Taylor and C.~G. Willcocks.
\newblock $\infty$-diff: Infinite resolution diffusion with subsampled mollified states.
\newblock \emph{arXiv preprint arXiv:2303.18242}, 2023.

\bibitem[Pidstrigach et~al.(2024)Pidstrigach, Marzouk, Reich, and Wang]{pidstrigach2024infinite}
J. Pidstrigach, Y. Marzouk, S. Reich, and S. Wang.
\newblock Infinite-dimensional diffusion models.
\newblock \emph{Journal of Machine Learning Research}, 25\penalty0 (414):\penalty0 1--52, 2024.

\bibitem[Wang(2008)]{wang2008karhunen}
L. Wang.
\newblock \emph{Karhunen-Loeve expansions and their applications}.
\newblock London School of Economics and Political Science (United Kingdom), 2008.

\end{thebibliography}

%%%%%%%%%%%%%%%%%%%%%%%%%%%%%%%%%%%%%%%%%%%%%%%%%%%%%%%%%%%%

%%%%%%%%%%%%%%%%%%%%%%%%%%%%%%%%%%%%%%%%%%%%%%%%%%%%%%%%%%%%

\newpage

\appendix

\section{Proofs of Section \ref{sec:gaussian}}

\subsection{Proof of Proposition \ref{prop:stationary-langevin}}\label{appendix:prop-3-1}

By Assumption \ref{assumption:score-error}, for the observed modes $j\leq N$, the preconditioned Langevin dynamics \eqref{eq:Langevin-SDE} reduces to the SDE
\begin{equation}
\begin{aligned}
dX_t^N =& -\left[ \operatorname*{Diag}\limits_{1\leq j \leq N}\left( s_j(\tau;\mu)\right) +  \operatorname*{Diag}\limits_{1\leq j \leq N}\left(\eps^a_j(\tau)\right) +    
C_N
\sigma^{-2} A_N^\top A_N \right] X_t^N dt\\
&
+ \left[  
C_N
\sigma^{-2} A_N^\top y -  \operatorname*{Diag}\limits_{1\leq j \leq N}\left(\eps_j^b(\tau)\right) \right] dt +
\sqrt{2C_N} dW^N_t, 
\end{aligned}
\label{eq:OU-obs}
\end{equation}
with $ C_N= \operatorname*{Diag}\limits_{1\leq j \leq N}\left(\lambda_j\right)$. For each unobserved mode $j>N$, we have  
\begin{align}
dX_t^{(j)} =& - \left[ s_j(\tau;\mu) +   \eps^a_j(\tau)\right] X^{(j)}_t  dt   - \eps_j^b(\tau) dt +\sqrt{2 
 \lambda_j} dW^{(j)}_t. 
 \label{eq:OU-unob}
\end{align}
Both \eqref{eq:OU-obs} and \eqref{eq:OU-unob} are Ornstein-Uhlenbeck (OU) processes. In particular,
$X^N_t \stackrel{t\to \infty}{\to} X^N_\infty$ in distribution, where the distribution of $X^N_\infty$ is the stationary distribution of \eqref{eq:OU-obs}:
\[ 
X^N_\infty \sim \mathcal{N}\left(\check{m}^N(\tau),\check{v}^N(\tau) \right),  
\]
with
\begin{align*}
&\check{m}^N(\tau)\\ & = \left[ \operatorname*{Diag}\limits_{1\leq j \leq N}\left( s_j(\tau;\mu)\right) +  \operatorname*{Diag}\limits_{1\leq j \leq N}\left(\eps^a_j(\tau)\right) +    
C_N
\sigma^{-2} A_N^\top A_N \right]^{-1} \! \! \left[  
C_N
\sigma^{-2} A_N^\top y -  \operatorname*{Diag}\limits_{1\leq j \leq N}\left(\eps_j^b(\tau)\right) \right] \\ 
%& = \left[ C_N^{-1} \operatorname*{Diag}\limits_{1\leq j \leq N}\left( s_j(\tau;\mu)\right) +  C_N^{-1} \operatorname*{Diag}\limits_{1\leq j \leq N}\left(\eps^a_j(\tau)\right) +   \sigma^{-2} A_N^\top A_N \right]^{-1}  C_N^{-1}  \left[  C_N \sigma^{-2} A_N^\top y -  \operatorname*{Diag}\limits_{1\leq j \leq N}\left(\eps_j^b(\tau)\right) \right]\\ 
& = \left[ C_N^{-1} \! \! \!\operatorname*{Diag}\limits_{1\leq j \leq N}\left( s_j(\tau;\mu)\right)\! + \! C_N^{-1}  \! \! \!\operatorname*{Diag}\limits_{1\leq j \leq N}\left(\eps^a_j(\tau)\right) \! + \!   
\sigma^{-2} A_N^\top A_N \right]^{-1}  \! \! \left[  
\sigma^{-2} A_N^\top y \!-  \!C_N^{-1}  \! \! \!\operatorname*{Diag}\limits_{1\leq j \leq N}\left(\eps_j^b(\tau)\right) \right],       
\end{align*}
and $\check{v}^N(\tau)$ is such that it solves the {Lyapunov equation}
\begin{align*}
&\left[ \operatorname*{Diag}\limits_{1\leq j \leq N}\left( s_j(\tau;\mu)\right) +  \operatorname*{Diag}\limits_{1\leq j \leq N}\left(\eps^a_j(\tau)\right) +    
C_N
\sigma^{-2} A_N^\top A_N \right] \check{v}^N(\tau) \\& + \check{v}^N(\tau) \left[ \operatorname*{Diag}\limits_{1\leq j \leq N}\left( s_j(\tau;\mu)\right) +  \operatorname*{Diag}\limits_{1\leq j \leq N}\left(\eps^a_j(\tau)\right) +    
C_N
\sigma^{-2} A_N^\top A_N \right]^\top  = 2C_N.    
\end{align*}
Then
\[
\check{v}^N(\tau) = \left[ C_N^{-1} \operatorname*{Diag}\limits_{1\leq j \leq N}\left( s_j(\tau;\mu)\right) +  C_N^{-1} \operatorname*{Diag}\limits_{1\leq j \leq N}\left(\eps^a_j(\tau)\right) +     
\sigma^{-2} A_N^\top A_N \right]^{-1},
\]
 since
\begin{align*}
&\left[ \operatorname*{Diag}\limits_{1\leq j \leq N}\left( s_j(\tau;\mu)\right) +  \operatorname*{Diag}\limits_{1\leq j \leq N}\left(\eps^a_j(\tau)\right) +    
C_N
\sigma^{-2} A_N^\top A_N \right] \check{v}^N(\tau)\\& = C_N \left( {\check{v}^N(\tau)}\right)^{-1} \check{v}^N(\tau) = C_N,
\end{align*}
and
\begin{align*}
&\check{v}^N(\tau) \left[ C_N^{-1} \operatorname*{Diag}\limits_{1\leq j \leq N}\left( s_j(\tau;\mu)\right) +  C_N^{-1} \operatorname*{Diag}\limits_{1\leq j \leq N}\left(\eps^a_j(\tau)\right) + \sigma^{-2} A_N^\top A_N \right] C_N  \\
&=  \check{v}^N(\tau)\left( {\check{v}^N(\tau)}\right)^{-1}C_N  = C_N.
\end{align*}
For each $j>N$, \eqref{eq:OU-unob} is a one-dimensional OU process with rate $s_j(\tau;\mu) + \eps_j^a(\tau)$, mean shift $-[s_j(\tau;\mu) + \eps_j^a(\tau)]^{-1}\eps_j^b(\tau)$, and noise $\sqrt{2\lambda_j}$. Hence 
$X^{(j)}_t \stackrel{t\to \infty}{\to} X^{(j)}_\infty$ in distribution, where the distribution of $X^{(j)}_\infty$ is the stationary  distribution of (\ref{eq:OU-unob})
\[
X^{(j)}_\infty \sim \mathcal{N} \left(-   \frac{\eps_j^b(\tau)}{s_j(\tau;\mu)   +  \eps^a_j(\tau)} ,   \frac{\lambda_j}{s_j(\tau;\mu) +  \eps^a_j(\tau)}  \right) .
\]
These results are valid as soon as $s_j(\tau;\mu) +  \eps^a_j(\tau)$ is a positive number for all $j$.

\subsection{Proof of Theorem \ref{thm:error-analysis}}\label{sec:thm:3-1}
For each mode $j$, define 
\[
\check{\mu}_j = \mu_j \left[e^{-\tau} + (1-e^{-\tau})p_j \right]. 
\]
The proof can be divided into two cases: one for the observed modes $j\leq N$, and one for the unobserved modes $j>N$. Since the KL divergence for the unobserved modes can be obtained by taking $A_N =0$ in the expression for the observed modes, we focus only on the latter. 

The marginal distributions of the $j$-th mode for the approximate and true posterior, for $j\leq N$, are given respectively by
\begin{align*}
\check{\pi}_y^{(j)}  = \mathcal{N}&\left( \left[ \frac{1}{\check{\mu}_j} + \sigma^{-2} (A^\top_N A_N)_{jj}   + \lambda_j^{-1} \eps_j^a(\tau) \right]^{-1}  \left[ \sigma^{-2} (A_N^\top y)_j  - \lambda_j^{-1} \eps_j^b(\tau) \right],\right.\\ &\left. \quad \! \! \left[ \frac{1}{\check{\mu}_j} + \sigma^{-2} (A^\top_N A_N)_{jj} + \lambda_j^{-1} \eps_j^a(\tau) \right]^{-1}\right) ,    
\end{align*}
and
\[
\pi_y^{(j)} = \mathcal{N}\left(\left[ \frac{1}{\mu_j}  +  \sigma^{-2}(A^\top_N A_N)_{jj}\right]^{-1}  \sigma^{-2} ( A_N^\top y)_j ,  \left[ \frac{1}{\mu_j}  +  \sigma^{-2}(A^\top_N A_N)_{jj}\right]^{-1} \right).
\]
The KL divergence between these two Gaussian distributions admits the explicit formula
\begin{align*}
 & \text{D}_\text{KL}\Big(\check{\pi}^{(j)}_y \,\Big|\Big|\, {\pi}^{(j)}_y \Big) \\ & =  \log \left( \left[ \frac{1}{\mu_j}  +  \sigma^{-2}(A^\top_N A_N)_{jj}\right]^{-1}  \left[ \frac{1}{\check{\mu}_j} + \sigma^{-2} (A_N^\top A_N)_{jj} + \lambda_j^{-1} \eps_j^a(\tau) \right]\right) \\
 & \quad - \frac{1}{2}  + \frac{1}{2} \left[ \frac{1}{\mu_j}  +  \sigma^{-2}(A^\top_N A_N)_{jj}\right]^{2} \left\{ \left[ \frac{1}{\check{\mu}_j} + \sigma^{-2} (A_N^\top A_N)_{jj} + \lambda_j^{-1} \eps_j^a(\tau) \right]^{-2} \right. \\& \quad \left. + \left( \left[ \frac{1}{\check{\mu}_j} + \sigma^{-2} (A^\top_N A_N)_{jj}   + \lambda_j^{-1} \eps_j^a(\tau) \right]^{-1}  \left[ \sigma^{-2} (A_N^\top y)_j - \lambda_j^{-1} \eps_j^b(\tau) \right] \right. \right. \\& \quad \left. \left. - \left[ \frac{1}{\mu_j}  +  \sigma^{-2}(A^\top_N A_N)_{jj}\right]^{-1}  \sigma^{-2}( A_N^\top y)_j \right)^2 \right\}   .  
\end{align*}
We study its limiting behavior as $\tau \to 0$. From the first term, we derive
\begin{equation}
\begin{aligned}
& \log \left( \left[ \frac{1}{\mu_j}  +  \sigma^{-2}(A^\top_N A_N)_{jj}\right]^{-1}  \left[ \frac{1}{\check{\mu}_j} + \sigma^{-2} (A_N^\top A_N)_{jj} + \lambda_j^{-1} \eps_j^a(\tau) \right]\right) \\
&=\log \left( \left[ 1  +  \sigma^{-2} \mu_j (A^\top_N A_N)_{jj}\right]^{-1}  \left[ \frac{1}{e^{-\tau}+(1-e^{-\tau})p_j} + \sigma^{-2} \mu_j (A_N^\top A_N)_{jj} + p_j^{-1} \eps_j^a(\tau) \right]\right)\\
& = \log \left( 1 + \frac{1}{1+\sigma^{-2}\mu_j(A^\top_N A_N)_{jj}} \left[ \frac{(1-e^{-\tau}) + (e^{-\tau} -1)p_j}{e^{-\tau} + (1-e^{-\tau}) p_j}  + p_j^{-1} \eps_j^a(\tau)\right] \right) \\
& = \log \left( 1 + \frac{1}{1+\sigma^{-2}\mu_j(A^\top_N A_N)_{jj}} \left[ -\tau(p_j-1) + p_j^{-1}\eps_j^a(\tau) + O(\tau^2) \right] \right) \\
& = \frac{1}{1+\sigma^{-2}\mu_j(A^\top_N A_N)_{jj}} \left( -\tau(p_j-1) + p_j^{-1}\eps_j^a(\tau) \right) + O(\tau^2).
\end{aligned}
\label{eq:T1-N}
\end{equation}
We now consider the other terms of the KL divergence. First, we derive
\begin{equation*}
\begin{aligned}
& \left[ \frac{1}{\check{\mu}_j} + \sigma^{-2} (A^\top_N A_N)_{jj}   + \lambda_j^{-1} \eps_j^a(\tau) \right]^{-1} \\ & \times  \left[ \sigma^{-2} (A_N^\top y)_j - \lambda_j^{-1} \eps_j^b(\tau) \right]  - \left[ \frac{1}{\mu_j}  +  \sigma^{-2}(A^\top_N A_N)_{jj}\right]^{-1}  \sigma^{-2} (A_N^\top y)_j\\
&  = \frac{\mu_j}{1+\sigma^{-2}\mu_j (A^\top_N A_N)_{jj}}  \Bigg\{ \Bigg[ 1 + \frac{1}{1+\sigma^{-2}\mu_j (A^\top_N A_N)_{jj}} \Bigg( \frac{(1-e^{-\tau}) + (e^{-\tau} -1)p_j}{e^{-\tau} + (1-e^{-\tau}) p_j} \\
& \quad  + p_j^{-1} \eps_j^a(\tau)\Bigg)  \Bigg]^{-1}   \left[ \sigma^{-2} (A_N^\top y)_j - \lambda_j^{-1} \eps_j^b(\tau) \right]  - \sigma^{-2} (A_N^\top y)_j   \Bigg\} \\
& = \frac{\mu_j}{1+\sigma^{-2}\mu_j (A^\top_N A_N)_{jj}}  \Bigg\{ \left[ 1 + \frac{1}{1+\sigma^{-2}\mu_j (A^\top_N A_N)_{jj}} \left( -\tau(p_j-1) + p_j^{-1}\eps^a_j(\tau) + O(\tau^2)  \right)  \right]^{-1} \\
&\quad \times \left[ \sigma^{-2} (A_N^\top y)_j - \lambda_j^{-1} \eps_j^b(\tau) \right]  - \sigma^{-2} (A_N^\top y)_j   \Bigg\} \\
& = \frac{\mu_j}{1+\sigma^{-2}\mu_j (A^\top_N A_N)_{jj}} \Bigg\{ \left[ 1 + \frac{1}{1+\sigma^{-2}\mu_j (A^\top_N A_N)_{jj}} \left( \tau(p_j-1) - p_j^{-1}\eps^a_j(\tau) + O(\tau^2)  \right)  \right] \\
&\quad \times \left[ \sigma^{-2} (A_N^\top y)_j - \lambda_j^{-1} \eps_j^b(\tau) \right]  - \sigma^{-2} (A_N^\top y)_j   \Bigg\}\\
& = \frac{\mu_j}{1+\sigma^{-2}\mu_j (A^\top_N A_N)_{jj}} \left\{  -\lambda_j^{-1} \eps_j^b(\tau) -\frac{\lambda_j^{-1} \eps_j^b(\tau)}{1+\sigma^{-2}\mu_j (A^\top_N A_N)_{jj}} \left( \tau(p_j-1) -p_j^{-1} \eps_j^a(\tau) \right)\right. \\
&\quad \left. + \frac{\sigma^{-2} (A_N^\top y)_j}{1+\sigma^{-2}\mu_j (A^\top_N A_N)_{jj}} \left( \tau(p_j-1) -p_j^{-1} \eps_j^a(\tau) \right)  \right\} + O(\tau^2).
\end{aligned}    
\end{equation*}
Taking the square yields
\begin{equation}
\begin{aligned}
& \left( \left[ \frac{1}{\check{\mu}_j} + \sigma^{-2} (A^\top_N A_N)_{jj}   + \lambda_j^{-1} \eps_j^a(\tau) \right]^{-1} \right. \\
&\times \left. \left[ \sigma^{-2} (A_N^\top y)_j + \lambda_j^{-1} \eps_j^b(\tau) \right] - \left[ \frac{1}{\mu_j}  +  \sigma^{-2}(A^\top_N A_N)_{jj}\right]^{-1}  \sigma^{-2} (A_N^\top y)_j\right)^2\\
& = \frac{\mu_j^2}{(1+\sigma^{-2}\mu_j (A^\top_N A_N)_{jj})^2}  \Bigg[  \lambda_j^{-2} \eps_j^b(\tau)^2 \\
&  \quad +  \frac{2\lambda_j^{-1} \eps_j^b(\tau)}{1+\sigma^{-2}\mu_j (A^\top_N A_N)_{jj}} \left( \lambda_j^{-1} \eps_j^b(\tau) -\sigma^{-2} (A_N^\top y)_j  \right)   \left( \tau(p_j-1) -p_j^{-1} \eps_j^a(\tau) \right)   \Bigg]  + O(\tau^2).
\end{aligned} 
\label{eq:T2-N}
\end{equation}
Next, we consider
\begin{equation}
\begin{aligned}
& \left[ \frac{1}{\check{\mu}_j} + \sigma^{-2} (A_N^\top A_N)_{jj} + \lambda_j^{-1} \eps_j^a(\tau) \right]^{-2}\\
&= \frac{\mu_j^2}{(1+\sigma^{-2}\mu_j (A^\top_N A_N)_{jj})^2} \!\left[ 1 + \frac{1}{1+\sigma^{-2}\mu_j (A^\top_N A_N)_{jj}} \left( -\tau(p_j-1) + p_j^{-1} \eps_j^a(\tau) + O(\tau^2)\right) \right]^{-2}\\
& = \frac{\mu_j^2}{(1+\sigma^{-2}\mu_j (A^\top_N A_N)_{jj})^2} \left[ 1 + \frac{2}{1+\sigma^{-2}\mu_j (A^\top_N A_N)_{jj}} \left( \tau(p_j-1) - p_j^{-1} \eps_j^a(\tau) \right) + O(\tau^2)\right]. 
\end{aligned}
\label{eq:T3-N}
\end{equation}
Putting \eqref{eq:T1-N}, \eqref{eq:T2-N}, and \eqref{eq:T3-N} together, we obtain
\begin{equation*}
\begin{aligned}
 &\text{D}_\text{KL}\Big(\check{\pi}^{(j)}_y \,\Big|\Big|\, {\pi}^{(j)}_y \Big)=   \frac{1}{2}\lambda_j^{-2} \eps_j^b(\tau)^2  \\
 &\quad -  \frac{\lambda_j^{-1} \eps_j^b(\tau)}{1+\sigma^{-2}\mu_j (A^\top_N A_N)_{jj}} \left(  \sigma^{-2} (A_N^\top y)_j  - \lambda_j^{-1} \eps_j^b(\tau)  \right) \left( \tau(p_j-1) -p_j^{-1} \eps_j^a(\tau) \right)  + O(\tau^2).
\end{aligned}
\end{equation*}

%We obtain
%\begin{align*}
%\text{D}_\text{KL}\Big(\check{\pi}^{(j)}_y \,\Big|\Big|\, {\pi}^{(j)}_y \Big)  =  \lambda_j^{-2}\eps_j^b(\tau)^2 \left(\frac{1}{2} + (p_j-1) \tau -p_j^{-1} \eps_j^a(\tau) \right) + O(\tau^2). 
%\end{align*}

\section{Proofs of Section \ref{sec:preconditioning}}

\subsection{Proof of Proposition \ref{prop:reversion-rate}}\label{appendix:prop-4-1}
If we assume that $A^\top_N A_N$ is diagonal in $(v_j)$, for the observed modes $j\leq N$ the preconditioned Langevin dynamics \ref{eq:Langevin-SDE} becomes
\begin{align}
\nonumber
dX_t^{(j)} =& - \left[ s_j(\tau;\mu) +   \eps^a_j(\tau) + \lambda_j \sigma^{-2} (A^\top_N A_N)_{jj} \right] X^{(j)}_t  dt   - \left[ \lambda_j \sigma^{-2} (A^\top_N y)_j -\eps_j^b(\tau)\right]  dt \\ &+\sqrt{2 
 \lambda_j} dW^{(j)}_t, 
 \label{eq:OU-ob-2}
\end{align}
while for each unobserved mode $j>N$, one obtains
\begin{align}
dX_t^{(j)} =& - \left[ s_j(\tau;\mu) +   \eps^a_j(\tau)\right] X^{(j)}_t  dt   - \eps_j^b(\tau) dt +\sqrt{2 
 \lambda_j} dW^{(j)}_t. 
 \label{eq:OU-unob-2}
\end{align}
Let $m^{(j)}(t) = \mathbb{E}[X_t^{(j)}]$. For the observed modes $j\leq N$, taking the expectation in the SDE above gives the linear ODEs
\begin{align*}
\frac{dm^{(j)}}{dt} =& - \left[ s_j(\tau;\mu) +   \eps^a_j(\tau) + \lambda_j \sigma^{-2} (A^\top_N A_N)_{jj} \right] m^{(j)}(t)     + \left[\lambda_j \sigma^{-2} (A^\top_N y)_j -\eps_j^b(\tau)\right] 
\end{align*}
while for each unobserved mode $j>N$  one obtains
\begin{align*}
\frac{dm^{(j)}}{dt} =& - \left[ s_j(\tau;\mu) +   \eps^a_j(\tau)\right] m^{(j)}(t)    - \eps_j^b(\tau) . 
\end{align*}
Given $m^{(j)}(0)=m^{(j)}_0$, both have unique solution. For $j\leq N$,
\begin{align*}
m^{(j)}(t) =& \left( m^{(j)}(0) - \frac{\lambda_j \sigma^{-2} (A^\top_N y)_j -\eps_j^b(\tau)}{s_j(\tau;\mu) +   \eps^a_j(\tau) + \lambda_j \sigma^{-2} (A^\top_N A_N)_{jj}}\right) e^{-\left[s_j(\tau;\mu) +   \eps^a_j(\tau) + \lambda_j \sigma^{-2} (A^\top_N A_N)_{jj}\right]t} \\& + \frac{\lambda_j \sigma^{-2} (A^\top_N y)_j -\eps_j^b(\tau)}{s_j(\tau;\mu) +   \eps^a_j(\tau) + \lambda_j \sigma^{-2} (A^\top_N A_N)_{jj}},
\end{align*}
which for $t\to \infty$ decays exponentially fast to the mean 
\[
\frac{\lambda_j \sigma^{-2} (A^\top_N y)_j -\eps_j^b(\tau)}{s_j(\tau;\mu) +   \eps^a_j(\tau) + \lambda_j \sigma^{-2} (A^\top_N A_N)_{jj}},
\]
with rate
\begin{align*}
\kappa^{(j)} & = s_j(\tau;\mu) +   \eps^a_j(\tau) + \lambda_j \sigma^{-2} (A^\top_N A_N)_{jj} \\ &
= \lambda_j   \left(  \left[e^{-\tau} \mu_j + (1-e^{-\tau}) \lambda_j\right]^{-1}   +   \sigma^{-2} 
%\sum_{i=1}^N A_{ij}^2 
(A_N^\top A_N)_{jj}
+  \lambda_j^{-1} \eps^a_j(\tau)  \right).
\end{align*}
For the unobserved modes $j>N$,
\begin{align*}
m^{(j)}(t) = \left( m^{(j)}(0) + \frac{\eps_j^b(\tau)}{s_j(\tau;\mu) +   \eps^a_j(\tau) }\right) e^{-\left[s_j(\tau;\mu) +   \eps^a_j(\tau) \right]t}  - \frac{ \eps_j^b(\tau)}{s_j(\tau;\mu) +   \eps^a_j(\tau) },
\end{align*}
which converge to the mean 
\[
 -\frac{ \eps_j^b(\tau)}{s_j(\tau;\mu) +   \eps^a_j(\tau) },
\]
with rate
\begin{align*}
\kappa^{(j)} & = s_j(\tau;\mu) +   \eps^a_j(\tau) \\ & = \lambda_j  \big[ 
\left[e^{-\tau} \mu_j + (1-e^{-\tau}) \lambda_j\right]^{-1}  +    \lambda_j^{-1} \eps^a_j(\tau) \big].
\end{align*}

\subsection{Proof of Theorem \ref{thm:preconditioner}}\label{sec:B-2}

We consider only the case of the observed modes $j\leq N$, since the case for the unobserved modes $j>N$ follows directly by setting $A_N =0$.

By Proposition \ref{prop:reversion-rate}, ensuring  uniform convergence rate for \eqref{eq:Langevin-SDE} using an approximate score function---as described in Assumption \ref{assumption:score-error}---amounts to solving the equation
\begin{equation}
\frac{\lambda_j}{\check{\mu}_j} + \lambda_j \sigma^{-2} (A^\top_N A_N)_{jj} + \eps_j^a(\tau) = 1, 
\label{eq:uniform-eq}
\end{equation}
where $\check{\mu}_j = \mu_j \left[e^{-\tau} + (1-e^{-\tau})p_j \right]$. 

Assume the expansions 
\[ \lambda_j = \lambda_j^{(0)} + \lambda_j^{(1)} \tau + O(\tau^2),\qquad \eps_j^a(\tau) = \eps_j^a \tau + O(\tau^2).\]
Then we compute
\[
\check{\mu}_j = \mu_j + \mu_j(p_j-1) \tau + O(\tau^2) = \mu_j + \lambda_j^{(0)} \tau -\mu_j \tau  +O(\tau^2),
\] 
using that $p_j = \lambda_j/\mu_j $
Substituting into \eqref{eq:uniform-eq}, we obtain that
$\lambda_j^{(0)}$ and $\lambda_j^{(1)}$ must satisfy
\begin{align*}
&\mu_j^{-1} ( {\lambda_j^{(0)}} + {\lambda_j^{(0)}} [1-\mu_j^{-1}\lambda_j^{(0)}]\tau + \lambda_j^{(1)}  \tau )  + \sigma^{-2} (A_N^\top A_N)_{jj} ( {\lambda_j^{(0)}} + \lambda_j^{(1)}  \tau )
+ \eps_j^a  \tau + O(\tau^2) = 1.
\end{align*}
Rearranging the terms, we get
${\lambda_j^{(0)}} \mu_j^{-1} + {\lambda_j^{(0)}} \sigma^{-2}(A^\top_N A_N)_{jj} =1$,
which gives
\[
\lambda_j^{(0)}=  \left[ \mu_j^{-1} + \sigma^{-2} (A^\top_N A_N)_{jj} \right]^{-1},
\]
and
\begin{align*}
\mu_j^{-1} \left( - {\lambda_j^{(0)}} \left[\mu_j^{-1}\lambda_j^{(0)}-1\right] + {\lambda_j^{(1)}} \right)  +\sigma^{-2} (A^\top_N A_N)_{jj} {\lambda_j^{(1)}}   + \eps_j^a  = 0,    
\end{align*}
yielding
\begin{align*}
{\lambda_j^{(1)}}  = \lambda_j^{(0)}  \left(  \mu_j^{-1}\lambda_j^{(0)} \left(\mu_j^{-1}\lambda_j^{(0)}-1\right) - \eps_j^a  \right).
\end{align*}

\section{Non-Gaussian Sampling: Technical Details}\label{sec:3:appendix}

\subsection{Proof of Proposition \ref{prop:score-phi}}
By $\Phi(X) = \sum_j \phi_j (X^{(j)})$,  the prior $\mu$ has Radon-Nikodym derivative with respect to the Gaussian $\mathcal{N}(0,C_\mu)$ given by
\[
\frac{d\mu}{d\mathcal{N}(0,C_\mu)} (X) \propto \prod_{j} \exp\left(-\phi_j(X^{(j)})\right).
\]
Since $C$ and $C_\mu$ are both diagonalized by the same basis $(v_j)$, the prior factorizes as a product of independent one-dimensional marginals in the coordinates $X^{(j)}$:
\[
\frac{d\mu}{d\mathcal{N}(0,C_\mu)} (X) = \prod_j \frac{d\mu^{(j)}}{d\mathcal{N}(0,\mu_j)} (X^{(j)}), 
\]
where
\[
\frac{d\mu^{(j)}}{d\mathcal{N}(0,\mu_j)} (X^{(j)}) \propto \exp\left(-\phi_j(X^{(j)})\right).
\]

We can then work mode by mode. For each $j$, define the one-dimensional OU process 
\[
\widetilde{X}_0^{(j)} \sim \mathcal{N}(0,\mu_j), \qquad \widetilde{X}_\tau^{(j)} = e^{-\tau/2} \widetilde{X}_0^{(j)} + \sqrt{1-e^{-\tau}} \xi^{(j)}, 
\]
with $\xi^{(j)} \sim \mathcal{N}(0,\lambda_j)$ independent of $\widetilde{X}_0^{(j)} $. Notice that
\begin{equation}
\begin{pmatrix}
\widetilde{X}^{(j)}_0\\\widetilde{X}^{(j)}_\tau    
\end{pmatrix}
\sim 
\mathcal{N}
\left( 0, 
\begin{pmatrix}
\mu_j & e^{-\tau/2} \mu_j\\
e^{-\tau/2} \mu_j & e^{-\tau} \mu_j +(1-e^{-\tau}) \lambda_j 
\end{pmatrix}
\right).
\end{equation}
The OU transition kernel is
\[
\tilde{p}(\widetilde{X}_\tau^{(j)} = x_\tau \mid \widetilde{X}_0^{(j)} = x_0) = \mathcal{N}(e^{-\tau/2} x_0, (1-e^{-\tau})\lambda_j) (x_\tau),
\]
where ${\cal N}(\mu,\sigma^2)(x)$ is the density at $x$ of the normal distribution with mean $\mu$ and variance $\sigma^2$.
%In fact, $\widetilde{X}_\tau^{(j)} \mid \widetilde{X}_0^{(j)}$ describes the transition kernel for the Ornestein-Uhlenback process towards the Gaussian $\mathcal{N}(0,\lambda_j)$. Notice also that
%\[
%p_{\widetilde{X}_0^{(j)}, \widetilde{X}_\tau^{(j)}} (x_0,x_\tau) = p_{\widetilde{X}_\tau^{(j)} \mid \widetilde{X}_0^{(j)}} (x_\tau \mid x_0) p_{\widetilde{X}_0^{(j)}}(x_0)  = \mathcal{N} (e^{-\tau/2} x_0, (1-e^{-\tau})\lambda_j)(x_\tau) \mathcal{N}(0,\mu_j)(x_0).
%\]
We push forward the prior $e^{-\phi_j} \mathcal{N}(0,\mu_j)$  through the OU kernel. Mode by mode, its density is
\begin{align*}
\mu_\tau^{(j)}(x_\tau) &\propto  \int  \exp(-\phi_j(x_0)) \mathcal{N}(0,\mu_j)(x_0) \tilde{p} (x_\tau \mid x_0) dx_0 \\
& = \int  \exp(-\phi_j(x_0)) \tilde{p}_{0,\tau} (x_0,x_\tau) d x_0, 
\end{align*}
where $\tilde{p}_{0,\tau} (x_0,x_\tau)$ denotes the joint density of $(\widetilde{X}_0^{(j)}, \widetilde{X}_\tau^{(j)})$. Let $\check{\mu}_j = e^{-\tau} \mu_j + (1-e^{-\tau})\lambda_j$. 
Dividing by the marginal Gaussian density of $\widetilde{X}_\tau^{(j)}$, we get 
\begin{align*}
\frac{\mu_\tau^{(j)}(x_\tau)}{\mathcal{N}(0,\check{\mu}_j )(x_\tau)} & \propto  \int  \exp(-\phi_j(x_0)) \frac{\tilde{p}_{0,\tau} (x_0,x_\tau)}{\tilde{p}_{\tau} (x_\tau)} d x_0 \\
& =  \int  \exp(-\phi_j(x_0))  \tilde{p} (x_0 \mid x_\tau) d x_0\\
& = \mathbb{E}[\exp(-\phi_j(\widetilde{X}_0^{(j)})) \mid \widetilde{X}^{(j)}_\tau = x_\tau].
\end{align*}
Since $ 
S^{(j)}(x_\tau,\tau; \mu_j) = \lambda_j \partial_j \log\frac{ \mu^{(j)}_\tau (x_\tau)}{\mathcal{N}(0, \check{\mu}_j)} + \lambda_j \mathcal{N}(0, \check{\mu}_j)(x_\tau), 
 $ we obtain
 
\[
S^{(j)}(x_\tau, \tau; \, \mu)  =  \lambda_j \partial_j \log \mathbb{E}[\exp(-\phi_j(\widetilde{X}_0^{(j)})) \mid \widetilde{X}^{(j)}_\tau = x_\tau]  -  
\frac{\lambda_j}{e^{-\tau} \mu_j + (1-e^{-\tau})\lambda_j} X^{(j)}.
\]

\subsection{Proof of Proposition \ref{prop:stationary-phi}}

For each mode $j$, define 
\[ \check{\phi}_j(X^{(j)},\tau) = - \log   \mathbb{E} \big[ \exp(-\phi_j(\widetilde{X}^{(j)}_0))\mid \widetilde{X}^{(j)}_\tau = X^{(j)} \big].
\]
Under Assumption \ref{assumption:score-error-phi}, the first $N$ coordinates of the preconditioned Langevin dynamics \eqref{eq:Langevin-SDE} corresponding to the observed modes $j \leq N$ satisfy
\begin{equation} 
\begin{aligned}
dX_t^N  =& - \left[ \operatorname*{Diag}\limits_{1\leq j \leq N}\left( 
 s_j(\tau;\mu)
 \right) + 
 C_N 
 \sigma^{-2}A_N^\top A_N +  \operatorname*{Diag}\limits_{1\leq j \leq N}\left(\eps_j^a(\tau)\right) \right]  X_t^N dt \\& 
 + \left[  
C_N \sigma^{-2} A_N^\top y - 
C_N \operatorname*{Diag}\limits_{1\leq j \leq N}\left(  \partial_j \check{\phi}_j(X^{(j)}_t)\right)  -   \operatorname*{Diag}\limits_{1\leq j \leq N}\left(\eps^a_j(\tau) \partial_j \phi_j (X^{(j)}_t)\right) \right. \\& \left. - \operatorname*{Diag}\limits_{1\leq j \leq N}\left(\eps_j^b(\tau)\right)\right] dt   
+
\sqrt{2C_N} dW^N_t.
\end{aligned}
\label{eq:prec-SDE-phi-obs}
\end{equation}
The SDE for the unobserved modes $j>N$ can be obtained by taking $A_N=0$  above:
\begin{equation}
\begin{aligned}
dX_t^{(j)}  =& - \left[ 
 s_j(\tau;\mu)
   +   \eps_j^a(\tau)  \right]  X_t^{(j)} dt   
 \\&  + \left[   - 
\lambda_j \partial_j \check{\phi}_j (X^{(j)}_t) -    \eps^a_j(\tau) \partial_j \phi_j (X^{(j)}_t)  -  \eps_j^b(\tau) \right] dt 
+
\sqrt{2\lambda_j} dW^{(j)}_t.
\end{aligned}
\label{eq:prec-SDE-phi-unobs}
\end{equation}
Both \eqref{eq:prec-SDE-phi-obs} and \eqref{eq:prec-SDE-phi-unobs} are preconditioned overdamped Langevin SDEs. In particular, one checks that \eqref{eq:prec-SDE-phi-obs} can be written as
\[
dX^N_t = -C_N \nabla U_N(X_t^N) dt + \sqrt{2C_N} dW_t^N,
\]
where the potential $U_N$ is
\begin{align*}
U_N(X^N) = & \frac{1}{2}{ X^N}^\top \left[ C_N^{-1} \operatorname*{Diag}\limits_{1\leq j \leq N}\left( 
 s_j(\tau;\mu)
 \right) +  
 \sigma^{-2}A_N^\top A_N +  C_N^{-1} \operatorname*{Diag}\limits_{1\leq j \leq N}\left(\eps_j^a(\tau)\right) \right] X^N \\ 
 & - \left[  
 \sigma^{-2} A_N^\top y  -  C_N^{-1}\operatorname*{Diag}\limits_{1\leq j \leq N}\left(\eps_j^b(\tau)\right)\right]X^N + 
\sum_{j=1}^N \left[  \check{\phi}_j(X^{(j)}_t)  -   \lambda_j^{-1} \eps^a_j(\tau)  \phi_j (X^{(j)}_t) \right].
\end{align*}
Its stationary distribution is $\check{\pi}_y^N$, which is absolutely continuous with respect to the Lebesgue measure over $\mathbb{R}^N$:
\[
%X^N_t \stackrel{t\to\infty}{\to}X_\infty^N\sim  \check{\pi}_y^N, \qquad 
\frac{d\check{\pi}_y^N(X^N)}{dX^N} \propto \exp(-  U_N(X_N) ). 
\]
We split $U_N$ into quadratic and non-quadratic terms. Hence
\[
\frac{d\check{\pi}_y^N(X^N)}{dX^N} \propto \exp\left(- \check{\Phi}^N(X^N,\tau) \right) {\cal N}(\check{m}^N(\tau), \check{v}^N(\tau))(X^N),
\]
where ${\cal N}(\check{m}^N(\tau), \check{v}^N(\tau))(X_N)$ is the density at $X_N$ of the multivariate Gaussian with mean $\check{m}^N(\tau)$ and covariance $ \check{v}^N(\tau)$, with
\begin{align*}
& \check{\Phi}^N(X^N,\tau) = \sum_{j=1}^N \left[ \check{\phi}_j(X^{(j)},\tau) +\lambda_j^{-1} \eps_j^a \phi_j (X_t^{(j)}) \right],\\
& \check{v}^N (\tau) = \left[ 
C_{N}^{-1} \operatorname*{Diag}\limits_{1\leq j \leq N}\left( s_j(\tau;\mu) \right) + \sigma^{-2} A_N^\top A_N + C_N^{-1} \operatorname*{Diag}\limits_{1\leq j \leq N} \left(  \eps_j^a(\tau)\right)  \right]^{-1},\\
&\check{m}^N(\tau)  = \check{v}^N(\tau)  \left[ \sigma^{-2} A_N^\top y  - C_N^{-1} \operatorname*{Diag}\limits_{1\leq j \leq N}\left(   \eps_j^b(\tau)\right) \right].
\end{align*}
By the same argument, for each $j>N$, the one-dimensional potential of \eqref{eq:prec-SDE-phi-unobs} is
\begin{align*}
U_j(X^{(j)})\! = \! \left[ 
 \lambda_j^{-1}s_j(\tau;\mu)
   +    \lambda_j^{-1}\eps_j^a(\tau)  \right]  \frac{{X_t^{(j)}}^2}{2}   
  +   \lambda_j^{-1}\eps_j^b(\tau) X_t^{(j)} +   
  \check{\phi}_j (X^{(j)}_t) +     \lambda_j^{-1}\eps^a_j(\tau)  \phi_j (X^{(j)}_t) . 
\end{align*}
Its stationary distribution is therefore
\[
\frac{d\check{\pi}_y^{(j)}(X^{(j)})}{dX^{(j)}} \propto \exp\left(- \check{\Phi}^{(j)}(X^{(j)},\tau) \right) {\cal N}(\check{m}^{(j)}(\tau), \check{v}^{(j)}(\tau))(X^{(j)} ),
\]
where ${\cal N}(\check{m}^{(j)}(\tau), \check{v}^{(j)}(\tau))(X^{(j)}) $ is the density at $X^{(j)} $ of the multivariate Gaussian with mean $\check{m}^{(j)}(\tau)$ and covariance $ \check{v}^{(j)}(\tau)$, with
\begin{align*}
& \check{\Phi}^{(j)}(X^{(j)},\tau) =  \check{\phi}_j(X^{(j)},\tau) +\lambda_j ^{-1} \eps_j^a \phi_j (X_t^{(j)}),\\
& \check{v}^{(j)} (\tau) = \left[ 
\lambda_j^{-1}  s_j(\tau;\mu)   + \lambda_j^{-1} \eps_j^a(\tau)   \right]^{-1},\\
&\check{m}^{(j)}(\tau)  = -\check{v}^{(j)}(\tau)   {\lambda_j}^{-1} \eps_j^b(\tau) .
\end{align*}

\subsection{Error Analysis in the Non-Gaussian Setting}\label{sec:non-gaussian-2}
For the sake of completeness, we present a result analogous to Theorem \ref{thm:error-analysis} for the non-Gaussian case. As the interested reader will notice, the calculations are significantly more involved, but remain relatively straightforward.

\begin{theorem}\label{thm:non-gaussian}
We assume $\eps_j^a(\tau) = O(\tau)$, $\eps_j^b(\tau) = O(1)$, and $A_N=\operatorname{Diag}_{1\leq j \leq N} (A_{jj})$. 
The Kullback-Leibler divergence between $\check{\pi}^{(j)}_y$ and ${\pi}^{(j)}_y$ is given by
\begin{align*}
& \textup{D}_\textup{KL}\Big(\check{\pi}^{(j)}_y \,\Big|\Big|\, {\pi}^{(j)}_y \Big) = B_j(\tau) + E_j(\tau),
\end{align*}
where $B_j(\tau)$ is a bias term given by
\begin{align*}
B_j(\tau) = & -2  \lambda_j^{-1} \eps_j^b(\tau)  \mathbb{E}_{\check{\pi}_y^{(j)}}\left[ x\right] + \lambda_j^{-1} p_j^{-1} (\eps_j^b(\tau))^2+\log  \int e^{-\phi_j(z)-\frac{1}{2 \sigma^2} \left[ A_{jj} z - y_j\right]^2} \mathcal{N}(0,\mu_j)(z) dz  \\
&- \log  \int e^{-\phi_j(z)-\frac{1}{2 \sigma^2} \left[ A_{jj} z - y_j\right]^2} \mathcal{N}(-p_j^{-1}\eps_j^b(\tau),\mu_j)(z) dz , 
\end{align*}
and $E_j(\tau)$ is an error term
\begin{align*}
E_j(\tau) = E_j^{(1)}(\tau) \tau + E_j^{(2)}(\tau) \lambda_j^{-1} \eps_j^a(\tau) + O(\tau^{3/2}), 
\end{align*}
where 
\begin{align*}
& E_j^{(1)}(\tau) = \left( \mathbb{E}_{\check{\pi}_y^{(j)}}\left[\frac{x^2}{\mu_j}\right] - \lambda_j^{-1} p_j^{-1}\eps_j^b(\tau)^2   \right) (1-p_j)  + \frac{1}{2}\mathbb{E}_{\check{\pi}_y^{(j)}} \left[ \lambda_j \left( \phi_j'(x)^2 - \phi_j''(x)\right) \right. \\
&  \quad \left. -\phi_j'(x) (1-2p_j)x   \right] - \mathcal{Z}(\eps_j^b(\tau),A_{jj},y_j)^{-1}   \int e^{-\phi_j(z)-\frac{1}{2 \sigma^2} \left[ A_{jj} z - y_j\right]^2} \mathcal{N}(-p_j^{-1}\eps_j^b(\tau),\mu_j)(z)  \\& \quad \times \Bigg[ \frac{1}{2}\Big(\lambda_j (\phi'(z)^2 -\phi''_j(z))     - \phi'_j(z) (1-2p_j) z \Big) + 
	\Bigg( \frac{z+p_j^{-1} \eps_j^b(\tau)}{\mu_j} \\& \quad  +   \frac{(z+p_j^{-1} \eps_j^b(\tau))^2}{\mu_j}-\frac{1}{2}\Bigg) (1-p_j)\Bigg]  dz, 
\end{align*}
and
\begin{align*}
& E_j^{(2)}(\tau) =  \mathbb{E}_{\check{\pi}_y^{(j)}}\left[x^2  - \phi_j(x)\right]   \\
& \quad  - p_j^{-2} \eps_j^b(\tau)^2 - \mathcal{Z}(\eps_j^b(\tau),A_{jj},y_j)^{-1}  \int e^{-\phi_j(z)-\frac{1}{2 \sigma^2} \left[ A_{jj} z - y_j\right]^2} \mathcal{N}(-p_j^{-1}\eps_j^b(\tau),\mu_j)(z) \\
& \quad \times \Bigg[ - \phi_j(z)  + 
	 z+p_j^{-1} \eps_j^b(\tau)  +   \left(z+p_j^{-1} \eps_j^b(\tau)\right)^2  -\frac{\mu_j}{2}  \Bigg]  dz,
\end{align*}    
with 
\[
\mathcal{Z}(\eps_j^b(\tau),A_{jj},y_j) =  \int e^{-\phi_j(z)-\frac{1}{2 \sigma^2} \left[ A_{jj} z - y_j\right]^2} \mathcal{N}(-p_j^{-1}\eps_j^b(\tau),\mu_j)(z) dz .
\]
\end{theorem}

\begin{proof}
In the following $\mathbb{E}_{\check{\pi}_y^{(j)}}\left[ \psi \right]$ and $\mathbb{E}_{\check{\pi}_y^{(j)}}\left[ \psi(x)\right]$ stand for $\int \psi(x) d\check{\pi}_y^{(j)}(x)$.
Recall that for $j\leq N$ the $j$-th mode marginal of the approximate posterior distribution $\check{\pi}_y$ is
\begin{align*}
	d\check{\pi}_y^{(j)}(X^{(j)}) = & \frac{1}{Z_{ \check{\pi}^{(j)}_y}}
	\exp\left( - \check{\phi}_j (X^{(j)}) - \lambda_j^{-1} \eps^a_j(\tau) \phi_j(X_t^{(j)}) - \frac{1}{2\sigma^2} \left[A_{jj} X^{(j)} - y_j \right]^2 \right) \\ &  \times d \mathcal{N}\left( -\left[ \frac{1}{\check{\mu}_j} + \lambda_j^{-1} \eps_j^a(\tau) \right]^{-1}  \lambda_j^{-1} \eps_j^b(\tau),  \left[ \frac{1}{\check{\mu}_j} + \lambda_j^{-1} \eps_j^a(\tau) \right]^{-1}\right) (X^{(j)}),    
\end{align*}
while the true posterior is
\[
d\pi^{(j)}_y(X^{(j)}) =\frac{1}{Z_{ \check{\pi}^{(j)}_y}} \exp\left(-\phi_j(X^{(j)})- \frac{1}{2\sigma^2} \left[A_{jj} X^{(j)} - y_j \right]^2 \right) d\mathcal{N}(0, \mu_j).
\]
For the unobserved modes $j>N$, we set $A_{jj}=0$. 
For each $j$, we have 
\begin{align*}
	&\text{D}_\text{KL}\Big(\check{\pi}^{(j)}_y \,\Big|\Big|\, {\pi}^{(j)}_y \Big)\\ 
    & = \underbrace{\mathbb{E}_{\check{\pi}_y^{(j)}}\left[ \log \mathcal{N}\left( -\left[ \frac{1}{\check{\mu}_j}+  \lambda_j^{-1} \eps_j^a(\tau) \right]^{-1}  \lambda_j^{-1} \eps_j^b(\tau),  \left[ \frac{1}{\check{\mu}_j} + \lambda_j^{-1} \eps_j^a(\tau) \right]^{-1}\right) - \log \mathcal{N}(0,{\mu}_j)\right]}_{\substack{\text{first  term}\\ \text{}}}\\
     & \quad + \underbrace{\mathbb{E}_{\check{\pi}^{(j)}_y} \Big([1 -\lambda_j^{-1} \eps^a_j(\tau) ] {\phi}_j - \check{\phi}_j  \Big)}_{\text{second term}}+  \underbrace{\log  \frac{Z_{ {\pi}^{(j)}_y}}{Z_{ \check{\pi}^{(j)}_y}}}_{\text{third term}}  .
\end{align*}
\paragraph{First term} We derive
\begin{align*}
	&\log \mathcal{N}\left( - \left[ \frac{1}{\check{\mu}_j} + \lambda_j^{-1} \eps_j^a(\tau) \right]^{-1}  \lambda_j^{-1} \eps_j^b(\tau),  \left[ \frac{1}{\check{\mu}_j} + \lambda_j^{-1} \eps_j^a(\tau) \right]^{-1} \right)(x)\\&= - \frac{1}{2} \log(2\pi) - \frac{1}{2} \log \left[ \frac{1}{\check{\mu}_j} + \lambda_j^{-1} \eps_j^a(\tau) \right]^{-1} \\ & \quad - \frac{1}{2} \left[ \frac{1}{\check{\mu}_j} + \lambda_j^{-1} \eps_j^a(\tau) \right] \left(x+ \left[ \frac{1}{\check{\mu}_j} + \lambda_j^{-1} \eps_j^a(\tau) \right]^{-1}  \lambda_j^{-1} \eps_j^b(\tau)\right)^2  ,
\end{align*}
and
\begin{align*}
	\log \mathcal{N}(0,\mu_j) (x)= -\frac{1}{2} \log(2\pi) - \frac{1}{2} \log \mu_j - \frac{1}{2} \frac{x^2}{\mu_j}.
\end{align*}
Hence
\begin{align*}
	& \mathbb{E}_{\check{\pi}_y^{(j)}}\left[ \log \mathcal{N}\left( -\left[ \frac{1}{\check{\mu}_j} + \lambda_j^{-1} \eps_j^a(\tau) \right]^{-1}  \lambda_j^{-1} \eps_j^b(\tau),  \left[ \frac{1}{\check{\mu}_j} + \lambda_j^{-1} \eps_j^a(\tau) \right]^{-1}\right) - \log \mathcal{N}(0,{\mu}_j)\right] \\
	& = \frac{1}{2} \log \left( \frac{1}{e^{-\tau} + (1-e^{-\tau}) p_j} + p_j^{-1} \eps_j^a(\tau)  \right) - \frac{1}{2} \mathbb{E}_{\check{\pi}_y^{(j)}} \left[ \frac{1}{\mu_j} \left( \frac{1}{e^{-\tau} + (1-e^{-\tau}) p_j}  + p_j^{-1} \eps_j^a(\tau) \right)  \right.  \\&\quad  \times \left.\left(x+ \left[ \frac{1}{e^{-\tau} + (1-e^{-\tau}) p_j} + p_j^{-1} \eps_j^a(\tau) \right]^{-1}  p_j^{-1} \eps_j^b(\tau)\right)^2 - \frac{x^2}{\mu_j} \right].
\end{align*}
We have
\begin{align*}
	&\frac{1}{2} \log \left( \frac{1}{e^{-\tau} + (1-e^{-\tau}) p_j} + p_j^{-1} \eps_j^a(\tau)  \right) = \frac{1}{2}\left((1-p_j)\tau  +  p_j^{-1} \eps_j^a (\tau)\right) +O(\tau^2),
\end{align*}
and
\begin{align*}
	&\mathbb{E}_{\check{\pi}_y^{(j)}} \left[\frac{1}{\mu_j} \left( \frac{1}{e^{-\tau} + (1-e^{-\tau}) p_j}  + p_j^{-1} \eps_j^a(\tau) \right) \right. \\
    &\qquad \quad \times \left. \left(x+ \left[ \frac{1}{e^{-\tau} + (1-e^{-\tau}) p_j} + p_j^{-1} \eps_j^a(\tau) \right]^{-1}  p_j^{-1} \eps_j^b(\tau)\right)^2 - \frac{x^2}{\mu_j} \right]
	\\
	&= \big[  (1-p_j)\tau + p_j^{-1} \eps_j^a(\tau)  \big] \mathbb{E}_{\check{\pi}_y^{(j)}}\left[ \frac{x^2}{\mu_j} \right]  +2  \lambda_j^{-1} \eps_j^b(\tau)  \mathbb{E}_{\check{\pi}_y^{(j)}}\left[ x\right] + \lambda_j^{-1} p_j^{-1} \eps_j^b(\tau)^2 \\
    &\quad \times \big[1-(1-p_j) \tau - p_j^{-1} \eps_j^a(\tau)\big] +O(\tau^2).
\end{align*}

\paragraph{Second term} 
We have
\[
\mathbb{E}\left[\exp\left(-\phi_j(\widetilde{X}_0)\right) \mid \widetilde{X}_\tau = x\right] = \int \exp(-\phi_j(z)) \frac{1}{\sqrt{2 \pi v_\tau}} \exp\left( - \frac{(z-m_\tau(x))^2}{2v_\tau} \right) dz,
\]
where
\begin{align*}
	m_\tau(x) & = \frac{e^{-\tau/2} \mu_j}{e^{-\tau} \mu_j + (1-e^{-\tau}) \lambda_j} x   = \left[ 1 +   \left( \frac{1}{2}-p_j \right)\tau + \left( \frac{p_j}{2} -\frac{1}{8} \right) \tau^2 \right] x  + O(\tau^3), 
\end{align*}
and
\begin{align*}
	\sqrt{v_\tau} &= \sqrt{\mu_j - \frac{e^{-\tau} \mu_j^2}{e^{-\tau} \mu_j + (1-e^{-\tau})\lambda_j}} =  \sqrt{\lambda_j\tau } \left(1-\frac{\tau}{4} - \frac{\tau^2}{32} + O(\tau^{5/2}) \right).
\end{align*}
By the change of variable $w = \frac{z-m_\tau(x)}{\sqrt{v_\tau}}$, we obtain
\begin{align*}
	 &\mathbb{E}\left[\exp\left(-\phi_j(\widetilde{X}_0)\right) \mid \widetilde{X}_\tau = x\right] 
     \\     
     & = \int \exp(-\phi_j(z)) \frac{1}{\sqrt{2 \pi v_\tau}} \exp\left( - \frac{(z-m_\tau(x))^2}{2v_\tau} \right) dz\\
	& =  \int \exp(-\phi_j \left(\sqrt{v_\tau} w + m_\tau(x)\right)) \frac{1}{\sqrt{2\pi}} \exp(-w^2/2) dw\\
	& =   \exp(-\phi_j(x)) \left[1 + \left\{\lambda_j \left[ \phi_j'(x)^2 - \phi_j''(x)\right] -\phi_j'(x) (1-2p_j)x \right\} \frac{\tau}{2} + O(\tau^{3/2})\right]
\end{align*}
where we used the Taylor expansion for $\exp(-\phi_j \left(\sqrt{v_\tau} w + m_\tau(x)\right))$ as $\sqrt{\tau} \to 0$
and 
\[
\int \frac{w}{\sqrt{2\pi}} \exp(-w^2/2) =0, \qquad \int \frac{w^2}{\sqrt{2\pi}} \exp(-w^2/2) =1.
\]
Hence
\begin{align*}
	& \mathbb{E}_{\check{\pi}_y^{(j)}}(\phi_j(x) - \check{\phi}_j(x)) \\
    & = \mathbb{E}_{\check{\pi}_y^{(j)}} \log \left[1 + \left\{\lambda_j \left[ \phi_j'(x)^2 - \phi_j''(x)\right] -\phi_j'(x) (1-2p_j)x \right\} \frac{\tau}{2} + O(\tau^{3/2})\right].
\end{align*}

\paragraph{Third term} For the unobserved modes $j>N$, we analyze
\begin{equation}
	\frac{\displaystyle \int e^{-\phi_j(z)} \mathcal{N}(0,\mu_j)(z)dz}{\displaystyle \int e^{-\check{\phi_j}(z) - \lambda_j^{-1} \eps_j^a(\tau) \phi_j(z)}  \mathcal{N}(-[\check{\mu}_j^{-1} + \lambda_j^{-1} \eps_j^a(\tau)]^{-1} \lambda_j^{-1} \eps_j^b(\tau), [\check{\mu}_j^{-1} + \lambda_j^{-1} \eps_j^a(\tau)]^{-1})(z)dz}.
	\label{eq:first-term}
\end{equation}
We use that
\[
\check{\phi}_j(z) = \phi_j(z) - \Big[ \lambda_j (\phi'_j(z)^2 - \phi''_j(z)) - \phi'_j(z) (1-2p_j) z\Big] \frac{\tau}{2} +O(\tau^{3/2}), 
\]
which implies
\begin{align*}
& e^{-\check{\phi_j}(z) - \lambda_j^{-1} \eps_j^a(\tau) \phi_j(z)} \\
& = e^{-\phi_j(z)} \Big( 1 + \Big[ \lambda_j (\phi'_j(z)^2 - \phi''_j(z)) - \phi'_j(z) (1-2p_j) z\Big] \frac{\tau}{2} -\lambda_j^{-1} \eps_j^a \phi_j(z) + O(\tau^{3/2}) \Big). 
\end{align*}
Now we consider the density of  $\mathcal{N}\big(-[\check{\mu}_j^{-1} + \lambda_j^{-1} \eps_j^a(\tau)]^{-1} \lambda_j^{-1} \eps_j^b(\tau), [\check{\mu}_j^{-1} + \lambda_j^{-1} \eps_j^a(\tau)]^{-1}\big)$:
\begin{equation}
	\frac{1}{\sqrt{2\pi [\check{\mu}_j^{-1} + \lambda_j^{-1} \eps_j^a(\tau)]^{-1}  }}
	\exp\left( - \frac{\Big(z+ [\check{\mu}_j^{-1} + \lambda_j^{-1} \eps_j^a(\tau)]^{-1}  \lambda_j^{-1} \eps_j^b(\tau)\Big)^2}{2 [\check{\mu}_j^{-1} + \lambda_j^{-1} \eps_j^a(\tau)]^{-1} } \right). 
	\label{eq:density}
\end{equation}
We have
\begin{equation}
	\frac{1}{\sqrt{2\pi [\check{\mu}_j^{-1} + \lambda_j^{-1} \eps_j^a(\tau)]^{-1}}} = \frac{1}{\sqrt{2 \pi \mu_j}} \Big[ 1 - \frac{1}{2} (1-p_j) \tau -\frac{1}{2} p_j^{-1} \eps_j^a(\tau) + O(\tau^2)\Big ].
	\label{eq:denominator}
\end{equation}
We now look at the exponent of \eqref{eq:density}. Its numerator reduces to
\[
(z+p_j^{-1} \eps_j^b(\tau))^2 - 2 (z+p_j^{-1} \eps_j^b(\tau)) ((1-p_j)\tau + p_j^{-1} \eps_j^a(\tau)) + O(\tau^2),
\]
while the reciprocal of its denominator  \eqref{eq:density} reduces to
\[
\frac{1}{2\mu_j} \Big[ 1+(1-p_j)\tau + p_j^{-1} \eps_j^a(\tau) + O(\tau^2) \Big]. 
\]
Then the exponent of \eqref{eq:density} can be expanded as
\begin{align*}
&-\frac{1}{2\mu_j} z^2 - \frac{1}{2\mu_j} p_j^{-2} \eps_j^b(\tau)^2 - z\lambda_j^{-1} \eps_j^b(\tau) \\
&+ \Bigg[ \frac{(z+p^{-1}_j \eps_j^b(\tau))}{\mu_j} + \frac{(z+p_j^{-1} \eps_j^b(\tau))^2}{2\mu_j}\Bigg ] ((1-p_j)\tau + p_j^{-1} \eps_j^a(\tau)) + O(\tau^2),
\end{align*}
and the exponential term in \eqref{eq:density} becomes
\begin{equation}
\begin{aligned}
	&\exp\left(-\frac{(z+p_j^{-1} \eps_j^b(\tau))^2}{2\mu_j}\right) \\
    &\times \Bigg[ 1 + \Bigg( \frac{z+p_j^{-1} \eps_j^b(\tau)}{\mu_j} + \frac{(z+p_j^{-1} \eps_j^b(\tau))^2}{2 \mu_j} \Bigg) ((1-p_j) \tau + p_j^{-1} \eps_j^a(\tau)) +O(\tau^2) \Bigg].
	\label{eq:density-2}
    \end{aligned}
\end{equation}
Putting together \eqref{eq:denominator} and \eqref{eq:density-2} we get that the Gaussian density \eqref{eq:density} is expanded as
\begin{align*}
&\mathcal{N}(-p_j^{-1} \eps_j^b(\tau), \mu_j)(z)\\ 
&\times  \Bigg[ 1 + \Bigg( \frac{z+p_j^{-1} \eps_j^b(\tau)}{\mu_j} +   \frac{(z+p_j^{-1} \eps_j^b(\tau))^2}{\mu_j} -\frac{1}{2}\Bigg) ((1-p_j)\tau + p_j^{-1} \eps_j^a(\tau)) + O(\tau^2) \Bigg].
\end{align*}
We can now expand for small $\tau$
\[
\left[ \int e^{-\check{\phi_j} - \lambda_j^{-1} \eps_j^a(\tau) \phi_j} d\mathcal{N}(-[\check{\mu}_j^{-1} + \lambda_j^{-1} \eps_j^a(\tau)]^{-1} \lambda_j^{-1} \eps_j^b(\tau), [\check{\mu}_j^{-1} + \lambda_j^{-1} \eps_j^a(\tau)]^{-1})\right]^{-1}. 
\]
Let
\[
\mathcal{Z}_j(\eps_j^b(\tau))=\int e^{-\phi_j(z)} \mathcal{N}(-p_j^{-1}\eps_j^b(\tau),\mu_j)(z) dz .
\]
We derive
\begin{align*}
	& \mathcal{Z}_j(\eps_j^b(\tau))^{-1} \Bigg\{ 1 - \mathcal{Z}_j(\eps_j^b(\tau))^{-1}  \\ 
	& \times \int  e^{-\phi_j(z)} \mathcal{N}(-p_j^{-1}\eps_j^b(\tau),\mu_j)(z) \Bigg[\Big(\lambda_j (\phi'(z)^2 -\phi''_j(z)) - \phi'_j(z) (1-2p_j) z \Big) \frac{\tau}{2} \\ &- \lambda_j^{-1} \eps_j^a(\tau) \phi_j(z)  + 
	\Bigg( \frac{z+p_j^{-1} \eps_j^b(\tau)}{\mu_j} +   \frac{(z+p_j^{-1} \eps_j^b(\tau))^2}{\mu_j} -\frac{1}{2}\Bigg) \bigg((1-p_j)\tau + p_j^{-1} \eps_j^a(\tau)\bigg)\Bigg] dz\Bigg\} \\ &+ O(\tau^{3/2}).
\end{align*}
Then \eqref{eq:first-term} can be expanded as
\begin{align*}
	& \log \mathcal{Z}_j(0) - \log \mathcal{Z}_j(\eps_j^b(\tau))  \\ 
	& - \mathcal{Z}_j(\eps_j^b(\tau))^{-1}\int e^{-\phi_j(z)} \mathcal{N}(-p_j^{-1}\eps_j^b(\tau),\mu_j)(z) \Bigg[\Big(\lambda_j (\phi'(z)^2 -\phi''_j(z)) \\&- \phi'_j(z) (1-2p_j) z \Big) \frac{\tau}{2}- \lambda_j^{-1} \eps_j^a(\tau) \phi_j(z)  + 
	\Bigg( \frac{z+p_j^{-1} \eps_j^b(\tau)}{\mu_j} +   \frac{(z+p_j^{-1} \eps_j^b(\tau))^2}{\mu_j} -\frac{1}{2}\Bigg)\\
    &\times ((1+p_j)\tau + p_j^{-1} \eps_j^a(\tau))\Bigg] dz + O(\tau^{3/2}) .
\end{align*}
Now let
\[
\mathcal{Z}_j(\eps_j^b(\tau), A_{jj},y_j) = \int e^{-\phi_j(z)-\frac{1}{2 \sigma^2} \left[ A_{jj} z - y_j\right]^2} \mathcal{N}(-p_j^{-1}\eps_j^b(\tau),\mu_j)(z) dz  .
\]
For the observed modes $j\leq N$, we get
\begin{align*}
	& \log   \mathcal{Z}_j(0, A_{jj},y_j)  - \log \mathcal{Z}_j(\eps_j^b(\tau), A_{jj},y_j)  \\ 
	& - \mathcal{Z}_j(\eps_j^b(\tau), A_{jj},y_j)^{-1}  \int  e^{-\phi_j(z)-\frac{1}{2 \sigma^2} \left[ A_{jj} z - y_j\right]^2} \mathcal{N}(-p_j^{-1}\eps_j^b(\tau),\mu_j)(z) \Bigg[ \Big(\lambda_j (\phi'(z)^2 -\phi''_j(z)) \\
    & - \phi'_j(z) (1-2p_j) z \Big) \frac{\tau}{2} - \lambda_j^{-1} \eps_j^a(\tau) \phi_j(z)  + 
	\Bigg( \frac{z+p_j^{-1} \eps_j^b(\tau)}{\mu_j} +   \frac{(z+p_j^{-1} \eps_j^b(\tau))^2}{\mu_j} -\frac{1}{2}\Bigg) \\
    &\times ((1-p_j)\tau + p_j^{-1} \eps_j^a(\tau))\Bigg] dz + O(\tau^{3/2}) .
\end{align*}	
	    
\end{proof}

\begin{remark}
If $\phi_j$ is smooth and $\eps_j^b(\tau)=O(1)$, then $E_j(\tau) \to 0$ as $\tau \to 0$. 
\end{remark}

\section{Illustrations: Additional Details}\label{sec:4:appendix}

Here we provide additional details on the theoretical setup underlying the illustrations. All illustrations were generated on Google Colab (13 GB of RAM), and all code executions took less than one minute\footnote{Code to reproduce results can be found at \href{https://github.com/balorenz1/SGM-Inf-Langevin}{https://github.com/balorenz1/SGM-Inf-Langevin}}.

\subsection{Recovering the KL coefficients of the Brownian sheet}
The Brownian sheet $B(x_1,x_2)$ is a Gaussian process with zero mean 
 and covariance
 \[
 \text{Cov}(B(x_1,x_2), B(y_1,y_2))=  \min(x_1,y_1) \min(x_2,y_2).
 \]
Its Karhunen-Lo\`eve expansion  \cite{wang2008karhunen} is
 \[
 B(x_1,x_2) = \sum_{j,k} \phi_{j,k} (x_1,x_2) \eta_{j,k}, \qquad (x_1,x_2)\in [0,1]^2,
 \]
 where $\eta_{j,k} \sim \mathcal{N}(0,\mu_{j,k})$ are independent Gaussian random variables, and
\[
\phi_{j,k}(x_1,x_2) = 2 \sin \left(\pi \left( j-\frac{1}{2} \right)  x_1\right) \sin \left(\pi\left( k-\frac{1}{2} \right)x_2\right),
\] 
\[
\mu_{j,k} =\left(\left(j-\frac{1}{2}\right) \pi\left(k-\frac{1}{2}\right)\pi\right)^{-2}.
\] 
 
 In Section \ref{sec:illustrations}, we truncate the KL expansion after $N$ modes
 \[
 B^N(x_1,x_2) = \sum_{j,k=1}^N \phi_{j,k}(x_1,x_2) \eta_{j,k},
 \]
and consider the inverse problem of recovering  the first $N^2$ coefficients from noisy observations corresponding to the first $M^2\leq N^2$ modes
\[ 
y_{j,k} = \tilde{\eta}_{j,k} + \eps_{j,k}, \qquad j,k\leq M,
\] 
where the prior is $\tilde{\eta}_{j,k}  \sim \mathcal{N}(0,\mu_{j,k})$ and the noise $\eps_{j,k} \sim \mathcal{N}(0,\sigma^2)$. This setup satisfies the assumptions of our theory, since the prior diagonal in the KL basis $(\phi_{j,k})$ and the forward map is simply the projection onto these modes, so that \[ 
A_{j,k,j',k'}=\delta_{j,j'} \delta_{k,k'},\qquad j,j', k', k\leq M ,\]
and zero otherwise. As a result, the posterior for each coefficient remains Gaussian 
\[
\widetilde{\eta}_{j,k} \mid y_{j,k} \sim \pi^{(j,k)}_{y_{j,k}} = \mathcal{N}(m_{j,k}, v_{j,k}),
\]
with, for $j,k \leq M$,
\[
v_{j,k} = \left( \mu_{j,k}^{-1} + \sigma^{-2} \right)^{-1} = \frac{\mu_{j,k} \sigma^2}{\mu_{j,k} + \sigma^2}, \qquad
m_{j,k} = \frac{\mu_{j,k}}{\mu_{j,k} + \sigma^2} \, y_{j,k},
\]
and for $j>M$ or $k>M$ (unobserved modes) the posterior simply coincides with the prior, $v_{j,k} =  \mu_{j,k}$, $m_{j,k} = 0$.

\paragraph{Experimental details } In Figure \ref{fig:brownian-sheet}, within the theoretical setup described above, we set the noise level $\sigma=10^{-2}$, chose $N=200$, and varied the number of observed modes $M^2=75^2,200^2$ to illustrate the discretization-invariance of the preconditioned Langevin sampler. This is confirmed by the small errors reported in the fourth column of Figure \ref{fig:brownian-sheet}. The preconditioned Langevin dynamics, using the preconditioner $C_M = \operatorname{Diag}\limits_{1\leq j,k \leq M}(\lambda_{j,k})$, $\lambda_{j,k} = [\mu_{j,k}^{-1} +\sigma^{-2}]^{-1}$, was run for $5\cdot10^3$ iterations with a fixed step-size of $5\cdot 10^{-1}$. We assumed access to the exact score function, i.e., $\phi=\tau =0$ in \eqref{def:score-non-gaussian}.

\subsection{Inverse source problem for the heat equation}
Let $\Omega = [0,1]^2 \subseteq \mathbb{R}^2$. Consider $u:\Omega \times[0,T] \to \mathbb{R}$ solving the heat equation
\[
\begin{cases}
\partial_t u (x,t) = \Delta u(x,t), & (x,t)\in \Omega \times (0,T], \\
u(x,0) = g(x), & x \in \Omega , \\
u(x,t) =0, & x \in \partial \Omega \times (0,T].
\end{cases}
\]
Set $u(x_1,x_2;t) \;=\; \sum_{j,k=1}^\infty  u_{j,k}(t)\,\psi_{j,k}(x_1,x_2)$,
where $(\psi_{j,k},\zeta_{j,k})$ are the Dirichlet eigenpairs of $-\Delta$ on $[0,1]^2$:
\[
\begin{cases}
-\Delta \psi_{j,k}(x_1,x_2) = \zeta_{j,k}\,\psi_{j,k}(x_1,x_2), & (x_1,x_2)\in[0,1]^2,\\
\psi_{j,k}|_{\partial [0,1]^2} =0 .
\end{cases}
\]
We have
\[
\psi_{j,k}(x_1,x_2) = 2\sin(j\pi x_1)\sin(k\pi x_2),
\qquad
\zeta_{j,k} = \pi^2(j^2 + k^2).
\]
The coefficients evolve as
\[
u_{j,k}(t) = e^{-\zeta_{j,k}t}\,g_{j,k},
\qquad
g_{j,k} = \langle g,\phi_{j,k}\rangle.
\]
In Section \ref{sec:illustrations}, we consider the the so-called {backward heat equation}---the ill-posed inverse problem of recovering the initial condition $g$  from noisy measurements of $u(\cdot,T)$ inside $\Omega$
\[
y_{j,k} = e^{-\zeta_{j,k}T}\,g_{j,k} +\eps_{j,k}, \qquad j,k \leq M,
\]
by adopting a Bayesian approach \cite{stuart2010inverse}. We assume a Gaussian prior
$g_{j,k}\sim\mathcal{N}(0,e^{-\beta \zeta_{j,k} })$
 and independent Gaussian noise $\eps_{j,k}\sim\mathcal{N}(0,\sigma^2)$. The forward map is diagonal in $(\psi_{j,k})$, with
\[
A_{j,k,j',k'} = e^{-\zeta_{j,k}T} \delta_{j,j'} \delta_{k,k'}, \qquad j,j',k,k' \leq M.
\]
As a result, the posterior for each coefficient remains Gaussian
\[
g_{j,k}\mid y_{j,k}
\sim \mathcal{N}(m_{j,k},v_{j,k}),
\]
with, for $j,k\leq M$,
\[
v_{j,k} = \frac{e^{-\beta \zeta_{j,k} }\sigma^2 }{e^{-(\beta+2T) \zeta_{j,k} } + \sigma^2 } , \qquad 
m_{j,k} = \frac{\mu_{j,k} }{e^{-(\beta+2T) \zeta_{j,k} }  + \sigma^2 } e^{- \zeta_{j,k} T }  y_{j,k}.
\]
For $j>M$ or $k>M$ (unobserved modes), the posterior simply coincides with the prior.

\paragraph{Experimental details} In Figure \ref{fig:heat-ip}, within the theoretical setup described above, we fixed the noise level at $\sigma=5 \cdot 10^{-3}$, chose $M=15$ (i.e. $225$ observed modes), and set $T=0.1$. We then ran the preconditioned Langevin sampler---with the optimal preconditioner $C$ from Theorem \ref{thm:preconditioner}---using the exact score function perturbed by a relative error $\eps_j^a \sim \mathcal{N}(0,0.1^2)$, scaled by $\tau=10^{-3}$ in the top row of Figure \ref{fig:heat-ip} and by $10^{-1}$ in its bottom row,  and with zero  bias (i.e. $\eps^b_j = 0$) to simulate a learned score. This sampler was run for $5\cdot 10^3$ iterations with a fixed step-size of $10^{-2}$. For comparison, we also executed the vanilla Langevin sampler for $1.5\cdot 10^4$ iterations with a fixed step-size of $10^{-6}$.
To further illustrate the quality of our preconditioned posterior samples, Figure \ref{fig:UQ} below shows uncertainty quantification for Figure \ref{fig:heat-ip}.  For the first $35$ modes, we plot the conditional posterior mean (red),  the 95\% credible interval (orange shading), and the ground truth (dotted black line).

\begin{figure}[ht]
    \centering
    \begin{subfigure}{0.49\linewidth}
        \centering
        \includegraphics[width=\linewidth]{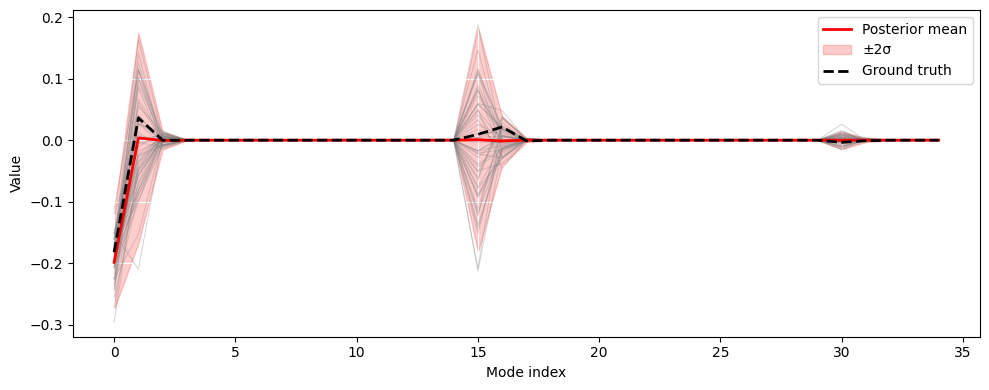}
    \end{subfigure}
    \hfill
    \begin{subfigure}{0.49\linewidth}
        \centering
        \includegraphics[width=\linewidth]{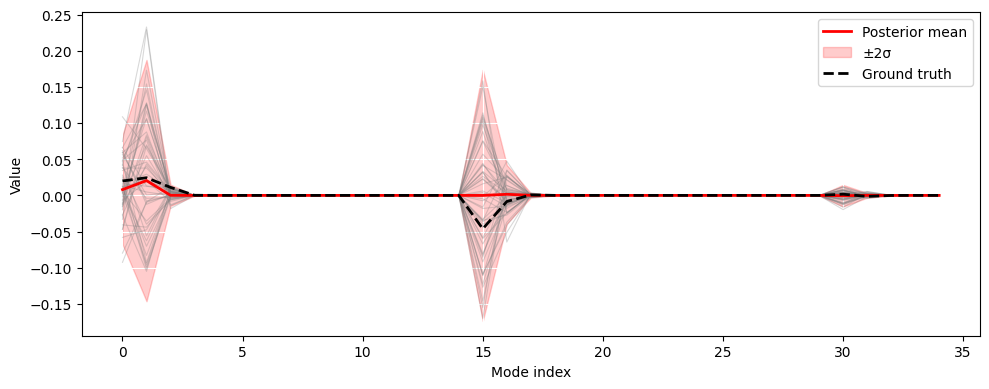}
    \end{subfigure}
   \caption{Uncertainty quantification for preconditioned posterior sampling. Left: $\tau=10^{-3}$. Right: $\tau=10^{-1}$}
    \label{fig:UQ}
\end{figure}

\end{document}